\documentclass{article}
\PassOptionsToPackage{table}{xcolor}

\usepackage{skillsmith_preprint,times}

\usepackage[utf8]{inputenc} 
\usepackage[T1]{fontenc}    
\usepackage[hidelinks]{hyperref}       
\usepackage{url}            
\usepackage{booktabs}       
\usepackage{amsfonts}       
\usepackage{nicefrac}       
\usepackage{microtype}      
\usepackage[table]{xcolor}  
\usepackage{amsmath}
\usepackage{graphicx}
\usepackage{algorithm, algorithmic}
\usepackage{multirow}
\usepackage[breakable]{tcolorbox}
\usepackage{fontawesome5}
\usepackage{amsmath,amssymb}
\usepackage{wrapfig}
\usepackage{caption}

\usepackage{colortbl}
\usepackage{pifont} 
\usepackage{graphicx}
\usepackage{booktabs}

\usepackage{amsthm}
\newtheorem{definition}{Definition}

\usepackage{amsmath, amssymb, amsthm}

\newtheorem{assumption}{Assumption}[section]

\newtheorem{analysislemma}{Lemma}[section]
\newtheorem{analysistheorem}[analysislemma]{Theorem}
\newtheorem{analysisprop}[analysislemma]{Proposition}
\newtheorem{analysiscor}[analysislemma]{Corollary}

\definecolor{HeaderBlue}{HTML}{2E4053} 
\definecolor{RowGray}{HTML}{F8F9F9}    

\definecolor{OursBule}{RGB}{230,240,255}
\definecolor{OursBg}{HTML}{E8F6F3}     
\definecolor{BestRed}{HTML}{C0392B}    
\definecolor{SecondBlue}{HTML}{2980B9} 

\definecolor{GainGreen}{HTML}{27AE60}  
\newcommand{\gain}[1]{\textbf{\textcolor{GainGreen}{+#1}}} 

\newcommand{\best}[1]{\textbf{\textcolor{BestRed}{#1}}}

\definecolor{myblue}{HTML}{4DBCC9}

\newcounter{methodprop}[section]
\renewcommand{\themethodprop}{\thesection.\arabic{methodprop}}

\definecolor{AbHeader}{HTML}{2E4053} 
\definecolor{AbOurs}{HTML}{E8F6F3}   
\definecolor{AbBest}{HTML}{C0392B}   

\newtcolorbox{bluebox}{
  colframe=myblue!50!white,
  colback=myblue!15!white,
  arc=0mm,
  boxrule=0mm,
  leftrule=1mm,
  left=1mm,
  right=1mm,
  top=1mm,
  bottom=1mm,
  breakable
}
\title{SkillSmith: Co-Evolving Skills and Tools for Self-Improving Agent Systems}

\author{%
\parbox{\textwidth}{\centering
\textbf{Yangbo Wei}$^{1,2}$ \quad
\textbf{Zhen Huang}$^{2,3}$ \quad
\textbf{Shaoqiang Lu}$^{1,2}$ \quad
\textbf{Junhong Qian}$^{4}$ \\
\textbf{Qifan Wang}$^{1,2}$ \quad
\textbf{Chen Wu}$^{2,5}$ \quad
\textbf{Lei He}$^{2}$\\[0.7em]
{\small\normalfont
$^{1}$Shanghai Jiao Tong University, Shanghai, China\\
$^{2}$Eastern Institute of Technology, Ningbo, China\\
$^{3}$University of Science and Technology of China, Hefei, China\\
$^{4}$Southeast University, Nanjing, China\\
$^{5}$Ningbo Institute of Digital Twin, Ningbo, China
}
}
}

\hypersetup{pdftitle={SkillSmith: Co-Evolving Skills and Tools for Self-Improving Agent Systems},pdfauthor={Yangbo Wei; Zhen Huang; Shaoqiang Lu; Junhong Qian; Qifan Wang; Chen Wu; Lei He}}

\begin{document}

\maketitle

\begin{abstract}
\setlength{\parfillskip}{0pt}
LLM agents increasingly store capabilities in external skill and tool libraries. Libraries can grow, but execution must fit a finite context window and improvement consumes a finite iteration budget. We formulate external-policy evolution under these two budgets and introduce \textsc{SkillSmith}. A skill ecosystem allocates context through replicator dynamics, using randomized masking and Bayesian estimation to learn skill effects and interactions. Four typed operations repair tools and move reusable logic across the skill--tool boundary, freeing context for the ecosystem. Conflict-driven learning distills failed proposals into version-scoped nogoods that guide repairs and veto repeated failing proposals. We prove utility ascent for fixed-model continuous dynamics, conditions under which lower context costs expand attainable utility, and the need for synchronized interface migration under persistent exact-match dependencies. Across five Qwen3.5 scales, SkillSmith improves over EvoSkill by up to 10.2 and 9.8 percentage points on OfficeQA and SealQA, respectively, and sustains improvement beyond SkillClaw on WildClawBench. \emph{Skills compete for context; tools are how the system makes room.}
\end{abstract}

\section{Introduction}

Tool use is a basic capability of LLM agents~\citep{nakano2021webgpt,karpas2022mrkl,yao2022react,schick2023toolformer}, and increasingly their capabilities reside outside model weights. A \emph{skill} packages workflow instructions, scripts, domain resources, and metadata into a reusable unit~\citep{zhang2026equipping}. Skills and tools together form an external policy layer that can be inspected, versioned, and transferred across models.

Methods for evolving this layer are emerging. Trace2Skill~\citep{ni2026trace2skill} and SkillX~\citep{wang2026skillx} distill skills from execution traces; EvoSkill~\citep{alzubi2026evoskill} optimizes them through failure-driven analysis and Pareto selection; SkillClaw~\citep{ma2026skillclaw} enables collective evolution across users. GEPA~\citep{agrawal2025gepa} shows that reflective textual evolution can be more sample-efficient than scalar-reward reinforcement learning.

Yet external capabilities can be read only through a finite context window. The library may grow, but the workbench remains fixed: each retrieved skill occupies space that others could use. Evolution must therefore manage competing capabilities, not merely accumulate them. This perspective reveals three problems (Figure~\ref{fig:motivation}).

\begin{figure}[t]
\centering
\captionsetup{skip=5pt}
\includegraphics[width=\linewidth]{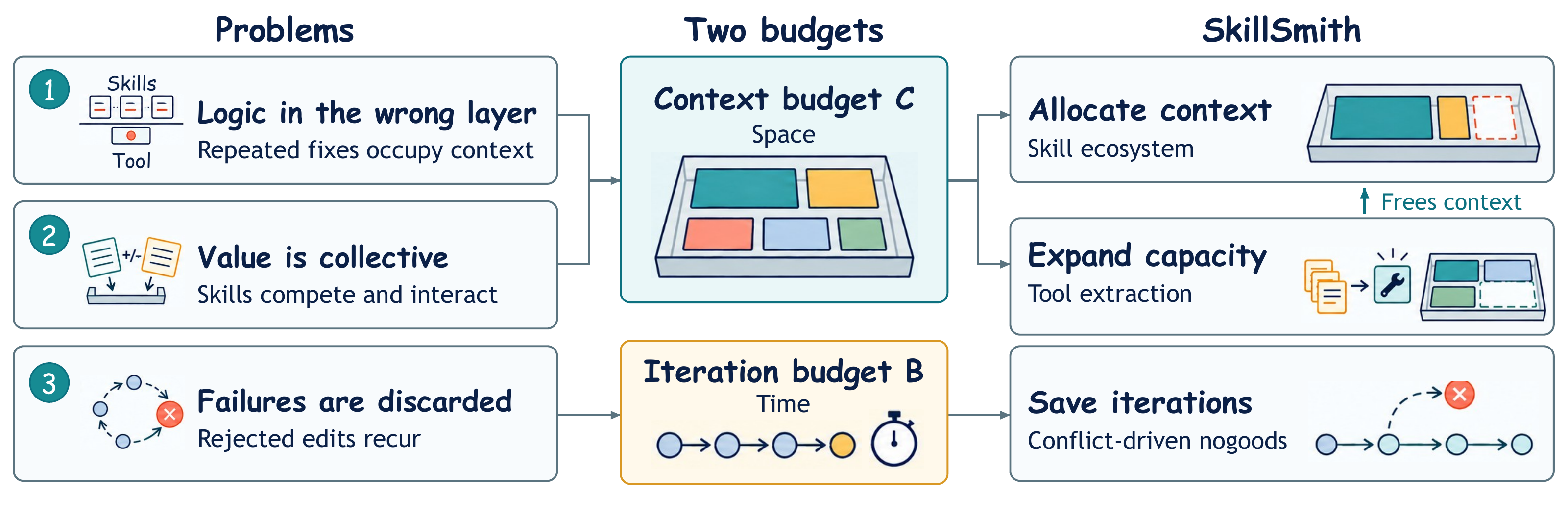}
\caption{\textbf{Three problems, two budgets, three mechanisms.} The skill ecosystem allocates context; tool extraction frees space for that allocation; conflict-driven nogoods save iterations. Extraction increases effective capacity within a fixed context window.}
\label{fig:motivation}
\vspace{-4pt}
\end{figure}

\noindent\ding{172}\,\textbf{Logic lives in the wrong layer.}
Holding tools fixed forces repairs into skills. Consider a table extractor that shifts one column at merged headers (Figure~\ref{fig:case_study}): three dependent skills each learn an offset correction for the same defect. All three corrections occupy context and rely on LLM interpretation, although the tool could implement the correction deterministically. Where logic resides thus determines both its context cost and reliability; skill-only evolution leaves this boundary fixed.

\noindent\ding{173}\,\textbf{Value is collective.}
Skills share context, and their value depends on their companions: two individually useful skills may interfere. Evaluating skills separately misses these interactions. Observational logs also confound a skill's effect with the tasks that retrieve it, while small training sets provide little evidence for individual skill pairs.

\noindent\ding{174}\,\textbf{Failures are discarded.}
When failed trajectories serve only one reflection and rejected proposals leave no persistent record, the same unsuccessful modification can recur, consuming iterations without adding knowledge.

We introduce \textsc{SkillSmith}, which co-evolves skills and tools under \emph{two budgets}: context limits space per task, while iterations limit improvement attempts. It connects three mechanisms through fault localization, which routes skill, tool, and interaction failures to the appropriate repair.

\noindent\textbf{Allocate context: skill ecosystem.}
Skills receive shares of reserved context, with fitness given by marginal utility per allocated token, including interactions. Shares follow replicator dynamics. Random masking identifies effects under a stationary second-order model, and Bayesian estimation shrinks poorly supported interactions toward zero. Thompson sampling supports exploration; posterior evidence governs retirement. For fixed effects and costs, continuous dynamics increase modeled utility and equalize marginal returns among positively allocated skills at equilibrium.

\noindent\textbf{Expand capacity: tool evolution.}
Shared, deterministic, context-independent logic belongs in tools; orchestration and judgment remain in skills. \textsc{Patch} and \textsc{Retype} repair tools, while \textsc{Extract} and \textsc{Inline} move logic across the boundary in opposite directions. Extraction replaces duplicated logic with tool calls, freeing context when the savings exceed interface and call-site overhead. This is niche construction for the skill ecosystem: tools make room for allocation. Interface changes and required caller edits are committed as atomic bundles to preserve legality.

\noindent\textbf{Save iterations: conflict-driven learning.}
Inspired by conflict-driven clause learning, SkillSmith uses delta debugging to reduce edit-induced failures to compact \emph{nogoods}. Stored evidence and remedies guide diagnosis; matching records veto repeated failures. Version scopes expire obsolete records as components evolve.

\noindent\textbf{Contributions.}
\ding{172}\,We formulate external-policy evolution under two budgets, making the skill--tool boundary an optimization variable.
\ding{173}\,SkillSmith unifies context allocation, capacity gains through tool evolution, and the reuse of failure evidence.
\ding{174}\,Our analysis establishes fixed-model allocation ascent, conditions for strict capacity gains, and synchronized migration requirements under persistent exact-match dependencies. Experiments span five model scales and three benchmarks.

\begingroup
\setlength{\parskip}{3pt}
\setlength{\abovedisplayskip}{4pt plus 1pt minus 1pt}
\setlength{\belowdisplayskip}{4pt plus 1pt minus 1pt}
\setlength{\abovedisplayshortskip}{0pt}
\setlength{\belowdisplayshortskip}{3pt}
\makeatletter
\renewcommand{\subsection}{\@startsection{subsection}{2}{\z@}
  {-1.0ex plus -0.2ex minus -0.1ex}{0.35ex}
  {\normalsize\sc\raggedright}}
\makeatother
\section{Problem Setup}\label{sec:problem_setup}
We fix the backbone model and evolve inspectable, versioned, transferable external assets accessed through a bounded context window.

\subsection{Three Ledgers}\label{sec:external_state}
The state $\Sigma_t=(\mathcal S_t,\mathcal T_t,\mathcal F_t)$ comprises skill, tool, and nogood libraries.

\noindent\textbf{Skills.}
Each skill $s=(m,w,r)\in\mathcal S_t$ packages metadata $m$ (name, trigger, version), workflow instructions $w$, and resources $r$ (templates, domain knowledge, or scripts).

\noindent\textbf{Tools.}
A tool $\tau=(d,f,\sigma)$ consists of an interface description $d$, implementation $f$, and signature $\sigma$ specifying argument types and output format. The library $\mathcal T_t$ defines available environment operations. Skills and tools calling $\tau$ form $\mathrm{Dep}(\tau)$ and $\mathrm{Dep}_{\mathcal T}(\tau)$, respectively, defining the dependency graph. A state is \emph{legal}, $\Sigma\in\Omega$, when every called tool exists and all call sites match its signature.

\noindent\textbf{Nogood memory.}
The memory $\mathcal F_t$ stores modification combinations verified to cause failure, scoped to component versions and paired with remedies (\S\ref{sec:nogood}). Skills and tools retain successful experience; nogoods retain failure evidence.

\noindent\textbf{Two execution modes.}
\ding{172}\,\textbf{Executor:} the LLM reads and interprets skill workflows; the machine executes tool implementations.
\ding{173}\,\textbf{Determinism:} skill execution is stochastic; tool execution is deterministic.
\ding{174}\,\textbf{Context:} skill content enters the window, whereas a tool exposes only its description $d$, keeping implementation $f$ outside context.
Skills carry orchestration and task-specific judgment; tools carry stable, testable logic. Evolving this boundary changes both cost and reliability (\S\ref{sec:tool_evolution}).

\subsection{Executing Through a Window}\label{sec:execution}
For task $x$, retrieval loads skills $\mathcal S_{\mathrm{act}}(x)\subseteq\mathcal S_t$ and their tools' interface descriptions. A skill's \emph{context cost} $c_i$ is its expected injected token count: metadata, workflow, loaded resources, and tool descriptions shared equally among callers. Costs stay fixed during allocation updates; execution also checks realized token counts when resource lengths vary. Retrieval enforces $\sum_{s_i\in\mathcal S_{\mathrm{act}}(x)}c_i\le C$. Execution produces a trajectory $\xi=(a_1,o_1,\ldots,a_L,o_L)$ and black-box score $P(x\mid\Sigma_t)\in[0,1]$. This score reflects workflows, tools, and their interactions; inferring which component caused a failure is a central challenge.

\subsection{Two Budgets}\label{sec:two_budgets}
Given initial state $\Sigma_0$, training and validation sets $D_{\mathrm{train}},D_{\mathrm{val}}$, and iteration budget $B$, each round proposes an atomic bundle $\mathcal L_t$ of skill edits and tool operations, with update $\Sigma_{t+1}=\mathcal U(\Sigma_t,\mathcal L_t)$. The objective is
\begin{equation}\label{eq:objective}
\max_{\{\mathcal L_t\}_{t=0}^{T-1}}\;
\mathbb E_{x\sim D_{\mathrm{val}}}\!\left[P(x\mid\Sigma_T)\right]
\quad\text{s.t.}\quad
\Sigma_t\in\Omega\ \ \forall t,\qquad
\sum_{t=0}^{T-1}b(\mathcal L_t)\le B,
\end{equation}
with every execution subject to context budget $C$. The cost $b(\mathcal L_t)$ counts executions spent proposing and validating the bundle, whether accepted or rejected.

\par
\endgroup

\section{SkillSmith}\label{sec:method}

Figure~\ref{fig:skillsmith_overview} shows how SkillSmith combines context allocation, tool evolution, and conflict-driven learning. We first define library utility and the skill ecosystem, then describe fault localization, typed tool operations, and failure memory. Section~\ref{sec:validation} details candidate validation; Appendix~\ref{app:loop_details} gives the full evolution algorithm.

\begin{figure}[t]
\centering
\includegraphics[width=\linewidth]{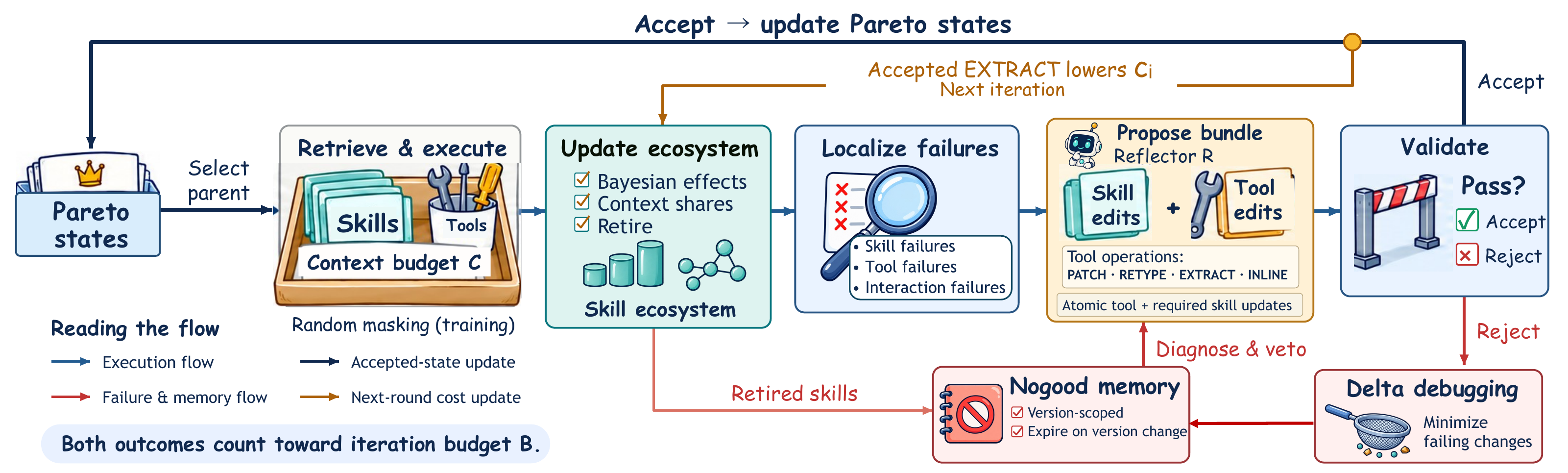}
\caption{SkillSmith allocates context through the skill ecosystem, frees space through tool extraction, and saves iterations through conflict-driven failure learning.}
\label{fig:skillsmith_overview}
\end{figure}

\subsection{Library Utility under a Context Budget}\label{sec:context_utility}
Let $p_i\in[0,1)$ be the nominal activation propensity of skill $s_i$. Under an independent-activation relaxation, a second-order model of the library's expected contribution to task residuals is
\begin{equation}\label{eq:library_utility}
\widetilde U(p)=\sum_i u_i p_i+\frac12\sum_{i\ne j}\beta_{ij}p_ip_j,
\end{equation}
where $u_i$ is the main effect and $\beta_{ij}=\beta_{ji}$ the interaction: positive for complementarity and negative for conflict, with $\beta_{ii}=0$. The propensities define an allocation surrogate; retrieval enforces the execution budget. External-state improvement can \textbf{(i)} revise skills to improve $u_i$ or $\beta_{ij}$, \textbf{(ii)} reallocate $p$, or \textbf{(iii)} reduce costs $c_i$. SkillSmith uses all three routes.

\subsection{Allocating Context: Skill Ecosystem}\label{sec:ecosystem}
\noindent\textbf{Population and fitness.}
Each skill receives reserved budget share $y_i$, with slack $y_0$ and $\sum_{i=0}^n y_i=1$. A smooth map converts reserved tokens $T_i=Cy_i$ into activation propensities:
\begin{equation}\label{eq:smooth_activation}
p_i(y)=1-e^{-Cy_i/c_i},\qquad T_i=-c_i\log(1-p_i).
\end{equation}
Since $1-e^{-z}\le z$, $\sum_i c_i p_i\le C(1-y_0)\le C$. Retrieval enforces the budget $C$ on each execution. Let $U(y)=\widetilde U(p(y))$. Fitness is
\begin{equation}\label{eq:fitness}
f_i(y)=\frac{\partial U}{\partial y_i}
=\frac{C}{c_i}(1-p_i)g_i(p),\qquad
 g_i(p)=u_i+\sum_{j\ne i}\beta_{ij}p_j,\quad f_0=0.
\end{equation}
Thus $f_i/C=g_i/(\partial T_i/\partial p_i)$ measures utility per marginal reserved token.

\noindent\textbf{Dynamics.}
Shares follow $\dot y_i=y_i(f_i-\bar f)$, with $\bar f=\sum_{i=0}^n y_if_i$. We use the discrete multiplicative-weights (Hedge) update~\citep{freund1997decision}:
\begin{equation}\label{eq:hedge}
y_i^{(t+1)}=\frac{y_i^{(t)}\exp(\eta\tilde f_i^{(t)})}
{\sum_{k=0}^n y_k^{(t)}\exp(\eta\tilde f_k^{(t)})}.
\end{equation}
Fitness $\tilde f$ uses a posterior draw. New skills receive $y_{\mathrm{init}}$, rescaling existing shares to preserve the simplex. Offspring inherit their parent Pareto state's population and posterior.

\refstepcounter{methodprop}\label{prop:replicator_ascent}
\noindent\textbf{Proposition~\themethodprop\ (allocation ascent).}
For fixed $u$, symmetric $\beta$, and positive $c,C$, the simplex is invariant under the continuous dynamics, and
\begin{equation}\label{eq:replicator_ascent}
\frac{dU}{dt}=\sum_{i=0}^n y_i(f_i-\bar f)^2\ge0,
\end{equation}
with equality exactly at stationarity; every limit point is stationary (Appendix~\ref{app:dynamics}).

\noindent\textbf{Equal marginal returns.}
At a stationary allocation, all skills with positive shares have the same marginal utility per reserved token. At a constrained optimum, inactive skills offer no larger marginal return; unused capacity implies a zero common return (Proposition~\ref{prop:kkt}). A sufficiently small Hedge step also increases $U$ when using true fitness; sampled updates incur an estimation-error term (Theorem~\ref{thm:hedge}). Appendix~\ref{app:dynamics} also gives a generalized Lotka--Volterra representation.

\noindent\textbf{Random masking and centered effects.}
During training, independently retain each retrieved skill with probability $q=1-\rho$. For retrieval set $R_e$ and retention indicator $a_{ei}$, use centered features $\kappa_{ei}=\mathbf 1[i\in R_e](a_{ei}-q)$. Validation and test runs are unmasked. Fit a Gaussian working model
\begin{equation}\label{eq:bayesian_effects}
z_e=a_{g(e)}+\sum_i\kappa_{ei}\gamma_i
       +\sum_{i<j}\kappa_{ei}\kappa_{ej}\beta_{ij}+\varepsilon_e,
\qquad \varepsilon_e\sim\mathcal N(0,\omega^2),
\end{equation}
where $a_{g(e)}$ is a task-category intercept. Under stationary, homogeneous second-order effects and positive retrieval exposure, centering identifies $\beta_{ij}$ despite arbitrary task/retrieval difficulty; it identifies $\gamma_i$ as the skill's average marginal contribution under the retrieval distribution (Theorem~\ref{thm:identification}). With $\pi_i=\Pr(i\in R)$ and $\pi_{ij}=\Pr(i,j\in R)$, recover the main effects and allocation fitness as
\begin{equation}\label{eq:centered_recovery}
u_i=\gamma_i-q\sum_{j\ne i}\beta_{ij}\frac{\pi_{ij}}{\pi_i},\qquad
g_i(p)=\gamma_i+\sum_{j\ne i}\beta_{ij}\left(p_j-q\frac{\pi_{ij}}{\pi_i}\right).
\end{equation}
Within each estimation window, co-retrieval frequencies estimate $\pi_{ij}/\pi_i$; posterior samples use the same recovery. Unretrieved skills retain their initialization prior until $\pi_i>0$.

\noindent\textbf{Bayesian estimation.}
Use priors $\gamma_i\sim\mathcal N(0,\tau_u^2)$ and $\beta_{ij}\sim\mathcal N(0,\tau_\beta^2)$ with $\tau_\beta<\tau_u$. Collect the centered features and category indicators in $\phi_e$, coefficients in $w$, and a proper Gaussian prior in $\mathcal N(\bar w_0,V_0)$. The working posterior has precision $V^{-1}=V_0^{-1}+\omega^{-2}\sum_e\phi_e\phi_e^\top$ and mean $\bar w=V(V_0^{-1}\bar w_0+\omega^{-2}\sum_e\phi_e z_e)$.
Observed co-retrieval pairs are fitted; unobserved pairs retain their prior. Noise variance is estimated from residuals. Before each round, forgetting inflates $V\leftarrow V+\xi I$. An edited skill inherits its predecessor's posterior, adding $\tau_u^2$ to its singleton variance and $\tau_\beta^2$ to its interaction variances. Proposition~\ref{prop:consistency} establishes posterior-mean consistency in a fixed stationary model without continual forgetting.

\noindent\textbf{Posterior-based decisions.}
\emph{Allocation and retrieval:} draw $\tilde w$ by Thompson sampling, transform it using Eq.~\ref{eq:centered_recovery}, and update $y$. For query $q_x$ and active set $A$, rank candidates by
$\alpha\,\mathrm{sim}(s_i,q_x)+\gamma\log p_i+\zeta\sum_{j\in A}\tilde\beta_{ij}$,
adding skills greedily while budget $C$ permits.
\emph{Retirement:} remove a skill if $\Pr(f_i>0\mid\mathrm{data})<\epsilon$, recording a nogood (\S\ref{sec:nogood}). With co-retrieval frequencies fixed, fitness is linear in the sampled coefficients, so this probability follows from the Gaussian posterior.
\emph{Mutation priority:} multiply $p_i$ by the suspiciousness defined next.

\subsection{Localizing Failures}\label{sec:proposer}
Failures may originate in a skill, a tool, or a skill interaction, requiring different repairs.
\textbf{Spectrum-based ranking.} For component $o\in\mathcal S\cup\mathcal T$, let $n_f(o)$ and $n_p(o)$ count failed and successful minibatch trajectories activating $o$, and let $N_f$ count all failed trajectories. Rank components by Ochiai suspiciousness~\citep{abreu2007accuracy}, $\mathrm{susp}(o)=n_f(o)/\sqrt{N_f(n_f(o)+n_p(o))}$.

\noindent\textbf{Confirmation and repair sites.}
Counterfactual checks on highly ranked components identify four site types:
\textbf{tool}, when call traces or unit tests reveal an implementation or contract failure;
\textbf{interaction}, when $\Pr(\beta_{ij}<0\mid\mathrm{data})>1-\epsilon$ but neither skill is individually suspicious;
\textbf{shared}, when one component contributes to failures involving at least two distinct skills, making it an \textsc{Extract} candidate (\S\ref{sec:tool_evolution}); and
\textbf{skill} otherwise.
The reflector $\mathcal R$ uses the ranked sites, confirming evidence, and matched nogood remedies to propose a repair. A flat suspiciousness distribution triggers reflection over complete trajectories.

\subsection{Expanding Context: Tool Evolution}\label{sec:tool_evolution}
\noindent\textbf{Extraction rule.}
Following library learning~\citep{ellis2021dreamcoder}, SkillSmith extracts reusable abstractions into tools when the logic is \textbf{shared} by multiple skills or recurs across their failures, \textbf{deterministic} enough to specify correctness through unit tests, and \textbf{context-independent}, requiring no task-specific judgment. New capabilities begin in skills and enter the tool layer once they meet these conditions; orchestration and judgment remain in skills.

\noindent\textbf{Extraction and capacity.}
Extracting logic of length $\ell$ shared by $k$ skills into a tool with interface length $\ell_d$ and per-skill call-site length $\ell_\delta$ changes each source skill's cost to
\begin{equation}\label{eq:extraction_cost}
c_i'=c_i-\ell+\ell_\delta+\ell_d/k.
\end{equation}
Costs fall when $\ell>\ell_\delta+\ell_d/k$; more sharing makes this condition easier to satisfy.

\refstepcounter{methodprop}\label{prop:capacity_monotonicity}
\noindent\textbf{Proposition~\themethodprop.}
Fix $u$ and $\beta$, and let $U^\star(c)=\max_{y\in\Delta}\widetilde U(p(y;c))$ under Eq.~\ref{eq:smooth_activation}. If $0<c'\le c$ componentwise, then $U^\star(c')\ge U^\star(c)$. The inequality is strict if an original optimum uses all reserved capacity with positive marginal value and at least one positively allocated skill becomes cheaper (Appendix~\ref{app:capacity}).

Behavior-preserving extraction thus frees capacity for the skill ecosystem: a form of niche construction~\citep{odlingsmee2003niche}.

\begin{wraptable}[9]{r}{0.50\textwidth}
\vspace{-\intextsep}
\centering
\small
\definecolor{OpInk}{HTML}{334155}
\definecolor{OpRepair}{HTML}{167D8D}
\definecolor{OpTransfer}{HTML}{99651F}
\captionsetup{skip=3pt,justification=raggedright,singlelinecheck=false}
\caption*{\textcolor{OpInk}{\textbf{Four typed operations}}}
\setlength{\tabcolsep}{4pt}
\renewcommand{\arraystretch}{0.94}
\arrayrulecolor{OpInk!30}
\begin{tabular*}{\linewidth}{@{\extracolsep{\fill}}lll@{}}
\textcolor{OpInk}{\textbf{\scriptsize Operation}} & \textcolor{OpInk}{\textbf{\scriptsize Effect}} & \textcolor{OpInk}{\textbf{\scriptsize Update}} \\
\midrule[0.4pt]
\rowcolor{OpRepair!9}
\multicolumn{3}{l}{\textcolor{OpRepair}{\scriptsize\bfseries Tool repair}} \\
\textcolor{OpRepair}{\textsc{Patch}} & Implementation & None \\
\textcolor{OpRepair}{\textsc{Retype}} & Signature & Callers \\
\addlinespace[1pt]
\rowcolor{OpTransfer!10}
\multicolumn{3}{l}{\textcolor{OpTransfer}{\scriptsize\bfseries Logic transfer}} \\
\textcolor{OpTransfer}{\textsc{Extract}} & New tool & Sources \\
\textcolor{OpTransfer}{\textsc{Inline}} & Remove tool & Callers \\
\bottomrule[0.4pt]
\end{tabular*}
\arrayrulecolor{black}
\end{wraptable}
\noindent\textbf{Typed operations.}
Tool-Smith $\mathcal B_\tau$ selects operations using localization and the extraction rule (Appendix~\ref{app:toolsmith}).
\textsc{Patch} fixes implementations under valid contracts; \textsc{Retype} revises contracts.
\textsc{Extract} lifts eligible shared logic into tools; \textsc{Inline} returns it to dependents when sharing is lost or task-dependent branches accumulate. Required updates cover source components or callers, including tools.

\noindent\textbf{Atomic bundles.}
A proposal $\mathcal L=(\text{tool operations},\text{skill edits})$ includes all required caller updates, including tool callers. Under persistent exact-match dependencies, legal single-component edits cannot complete a breaking interface migration (Proposition~\ref{prop:interface}); \textsc{Retype} and \textsc{Inline} therefore commit atomically with dependent edits. Static checks cover changed components and required callers (Proposition~\ref{prop:locality}); tests assess behavior. Source versions enable exact rollback; nogoods share the operation vocabulary (\S\ref{sec:nogood}).

\noindent\textbf{Acceptance.}
Bundles must pass validation without regression (\S\ref{sec:validation}). Among acceptable bundles, prefer larger $\Delta\mathrm{perf}-\lambda_{\mathrm{cost}}\Delta\bar c$, where $\Delta\bar c$ measures the change in expected context tokens per task, including interface and call-site overhead. Accuracy-preserving extraction can therefore be accepted for freeing context; increased inlining costs require a compensating performance gain.

\subsection{Saving Iterations: Conflict-Driven Learning}\label{sec:nogood}
Inspired by conflict-driven learning in SAT solvers~\citep{marquessilva1999grasp}, memory $\mathcal F_t$ stores \emph{nogoods} that prune previously failed library modifications.
\textbf{Representation.} A nogood $\nu=(A,\mathrm{scope},r)$ records a failure-inducing set of modifications $A$, each described by its component and structured change; the relevant component versions as its scope; and a remedy $r$, either a repair or an \texttt{open} marker.

\noindent\textbf{Learning causes.}
For a validation test that the parent passes and a proposed bundle fails, delta debugging~\citep{zeller2002simplifying} isolates a smaller failing subset $A$ under deterministic replay. Completed minimization yields a \emph{1-minimal} set: deleting any single edit no longer reproduces that failure. Bundles usually contain at most five modifications; the $n_{\mathrm{dd}}$ rerun cap may stop before this guarantee is reached (Lemma~\ref{lem:ddmin}). Retirement and \textsc{Inline} additionally record reconstruction warnings, distinguished from replay-verified failures.

\noindent\textbf{Diagnosis, veto, and expiration.}
Retrieve similar nogoods and pass their remedies to $\mathcal R$. Veto a proposal when it matches every modification in an active nogood's $A$: use exact matching for structural changes and embedding similarity with threshold $\theta_{\mathrm{veto}}$ for descriptive changes. Return a vetoed proposal to $\mathcal R$ with the remedy as a constraint; discard it after two unsuccessful revisions. Any version change within a nogood's scope expires it: expired records support diagnosis but cannot veto. A proposal matching an expired record proceeds to validation, whose outcome refreshes or removes that record. Exact-match veto is sound under version-scoped failure persistence; remedial edits and similarity matches need not satisfy this assumption (Proposition~\ref{prop:veto}).

\subsection{Validation and Failure Learning}\label{sec:validation}
\begin{wrapfloat}{algorithm}{r}{0.51\textwidth}
\small
\caption{Validate, commit, or learn}\label{alg:bundle_validation}
\begin{algorithmic}[1]
\REQUIRE $\Sigma,\mathcal L,\mathcal G,\mathcal F,B_{\mathrm{rem}}$
\STATE $\Sigma'\gets\textsc{Stage}(\Sigma,\mathcal L)$
\STATE $(v,t)\gets\textsc{Check}(\Sigma',B_{\mathrm{rem}})$
\IF{$v=\mathrm{pass}$ and $\Sigma'$ is acceptable}
\STATE $\mathcal G\gets\mathcal G\cup\{\Sigma'\}$; \textbf{return} accept
\ENDIF
\IF{$v=\mathrm{fail}$ and $t(\Sigma)=\mathrm{pass}$}
\STATE $A\gets\textsc{ddmin}(\mathcal L,t,\min(n_{\mathrm{dd}},B_{\mathrm{rem}}))$
\STATE $\nu\gets(A,\textsc{Scope}(\Sigma,A,t),\texttt{open})$
\STATE $\mathcal F\gets\mathcal F\cup\{\nu\}$
\ELSE
\STATE Retain diagnostic evidence only
\ENDIF
\RETURN reject; preserve $(\mathcal S,\mathcal T)$
\end{algorithmic}
\end{wrapfloat}
After proposal revision and nogood screening, Algorithm~\ref{alg:bundle_validation} validates candidates and learns from rejected edits.

\noindent\textbf{Validate before committing.}
Stage the bundle on a parent copy. \textsc{Check} verifies dependency closure and call-site compatibility on affected components, then runs tool unit, integration, and regression tests on $D_{\mathrm{val}}$. It returns outcome $v$ and the first failing test $t$, if any. Validated candidates enter $\mathcal G$ only if accepted under Section~\ref{sec:tool_evolution}. Accepted extraction updates costs for the next ecosystem update; rejection preserves skills, tools, and costs.

\noindent\textbf{Learn only reproducible causes.}
Replay $t$ on the parent before minimization. If it passes, capped delta debugging yields a failing subset $A$, scoped to parent versions of edited components and relevant replay dependencies. The remedy is initially \texttt{open}. Otherwise, retain diagnostics without asserting an edit-induced cause. All tests, including baseline replay and minimization, debit the shared budget $B_{\mathrm{rem}}$. Exhaustion yields an unresolved outcome, never acceptance; minimization retains its last confirmed failing subset.

\section{Experiments}

We organize the experiments around four questions.
$\blacktriangleright$ \textbf{RQ1 (Performance):} How does skill--tool co-evolution compare with skill-only evolution?
$\blacktriangleright$ \textbf{RQ2 (Mechanisms):} What do typed tool evolution, the skill ecosystem, and conflict-driven nogood learning each contribute?
$\blacktriangleright$ \textbf{RQ3 (Scaling):} How do gains vary with task complexity, skill co-activation, and context demand?
$\blacktriangleright$ \textbf{RQ4 (Stability):} How do the mechanisms support recovery, control library growth, and reduce regressions during continued evolution?

\subsection{Experimental Setup}

\paragraph{Benchmarks.}
We evaluate \textsc{SkillSmith} on three benchmarks.
\textbf{OfficeQA} requires cross-document table localization and multi-step numerical reasoning, representing structured document scenarios that are highly sensitive to skill precision.
\textbf{SealQA} examines factual question-answering ability when search results from the open web contain noise and conflicting information.
\textbf{WildClawBench} simulates real-world agent deployment, featuring 15--50 step multi-modal tasks with multiple tools and high interaction density, making it a suitable testbed for evaluating joint Skill--Tool co-evolution.

\paragraph{Baselines.}
We select three methods representing different evolution paradigms as our primary comparisons:
\textbf{Base Agent} (initial skill library $\mathcal{S}_0$ + fixed tool set $\mathcal{T}_0$, no evolution),
\textbf{EvoSkill} (failure-driven Skill-only evolution + Pareto front selection),
and \textbf{SkillClaw} (multi-user trajectory-driven collective Skill evolution).
Reflection uses the same frozen backbone as execution at each scale, as in GEPA~\citep{agrawal2025gepa} and EvoSkill~\citep{alzubi2026evoskill}. Data splits, initial skill and tool libraries, scoring functions, and additional evaluation details are provided in Appendix~\ref{app:exp_details}.

\begin{figure*}[t]
    \centering
    \includegraphics[width=\textwidth]{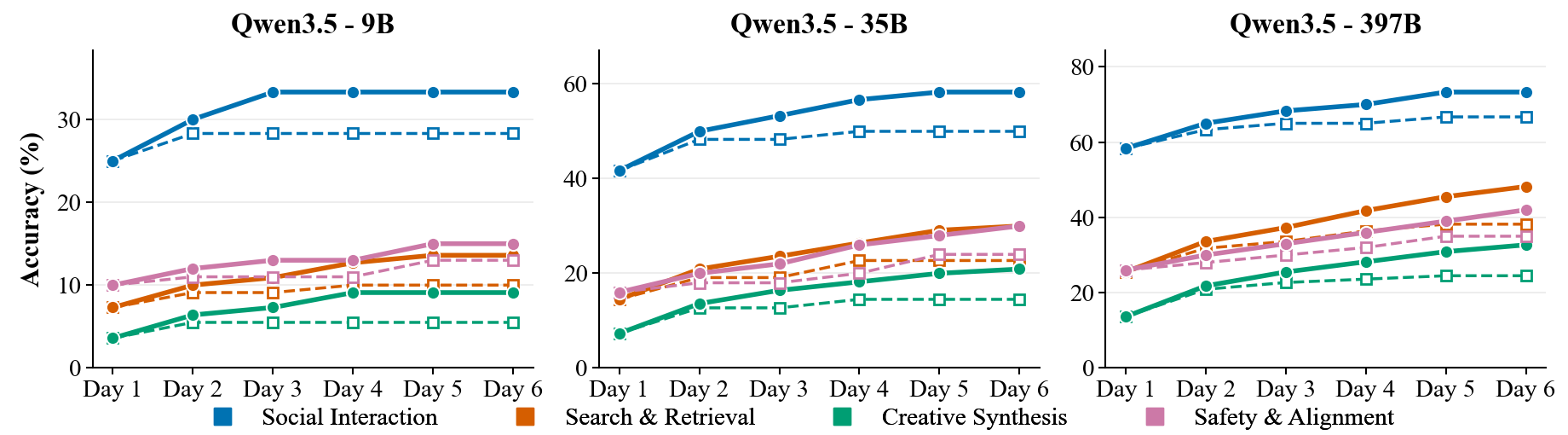} 

    \caption{\textbf{Evolutionary Trajectories on WildClawBench.} Accuracy over 6 days across three Qwen3.5 scales. SkillSmith (solid) continues improving through skill--tool co-evolution, while the fixed-tool SkillClaw baseline (dashed) plateaus earlier.}
    \label{fig:wildclaw_trajectory}

\end{figure*}

\begin{table*}[t]
\centering
\caption{\textbf{Main results on OfficeQA and SealQA.} Accuracy (\%) across Qwen3.5 scales (mean $\pm$ sample SD over five runs). \textbf{\textcolor{BestRed}{Red}} denotes best and \textcolor{SecondBlue}{blue} second best. Gains are in percentage points.}
\label{tab:main_results}
\begingroup
\small
\newcommand{\resultstd}[2]{#1\,{\scriptsize$\pm$#2}}
\renewcommand{\arraystretch}{1.16}
\setlength{\tabcolsep}{4pt}
\resizebox{\textwidth}{!}{%
\begin{tabular}{@{}l ccccc @{\hspace{10pt}} cccc@{}}
\toprule
& \multicolumn{5}{c}{\textbf{OfficeQA} (205 questions)} & \multicolumn{4}{c}{\textbf{SealQA} (92 questions)} \\
\cmidrule(lr){2-6}\cmidrule(lr){7-10}
\textbf{Method} & 9B & 27B & 35B & 122B & 397B & 9B & 35B & 122B & 397B \\
\midrule
Base Agent & \resultstd{14.1}{0.5} & \resultstd{26.3}{0.8} & \resultstd{34.6}{0.8} & \resultstd{52.7}{1.1} & \resultstd{74.1}{0.8} & \resultstd{5.4}{0.8} & \resultstd{13.0}{1.1} & \resultstd{20.7}{1.7} & \resultstd{37.0}{1.7} \\
EvoSkill & \resultstd{\textcolor{SecondBlue}{16.1}}{0.5} & \resultstd{\textcolor{SecondBlue}{30.7}}{0.8} & \resultstd{\textcolor{SecondBlue}{41.0}}{0.8} & \resultstd{\textcolor{SecondBlue}{63.9}}{1.1} & \resultstd{\textcolor{SecondBlue}{85.9}}{0.5} & \resultstd{\textcolor{SecondBlue}{6.5}}{0.8} & \resultstd{\textcolor{SecondBlue}{17.4}}{1.1} & \resultstd{\textcolor{SecondBlue}{28.3}}{1.7} & \resultstd{\textcolor{SecondBlue}{50.0}}{1.7} \\
\textbf{SkillSmith (Ours)} & \resultstd{\best{17.6}}{0.5} & \resultstd{\best{34.1}}{0.8} & \resultstd{\best{46.3}}{0.8} & \resultstd{\best{72.2}}{0.8} & \resultstd{\best{96.1}}{0.3} & \resultstd{\best{7.6}}{0.8} & \resultstd{\best{20.7}}{1.1} & \resultstd{\best{33.7}}{1.7} & \resultstd{\best{59.8}}{1.7} \\
\midrule
\rowcolor{OursBg}
$\Delta$ vs Base & \gain{3.4} & \gain{7.8} & \gain{11.7} & \gain{19.5} & \gain{22.0} & \gain{2.2} & \gain{7.6} & \gain{13.0} & \gain{22.8} \\
\bottomrule
\end{tabular}}
\endgroup
\end{table*}

\subsection{RQ1: Main Results and RQ2: Ablation Study}

\paragraph{\textbf{Main Results.}} As shown in Table \ref{tab:main_results}, SkillSmith consistently outperforms both the Base Agent and the strong skill-only baseline, EvoSkill, on OfficeQA and SealQA. The gains increase with model scale: as model size increases from 9B to 397B, SkillSmith's gains over the Base Agent grow from +3.4 to +22.0 percentage points (OfficeQA) and +2.2 to +22.8 percentage points (SealQA), showing larger absolute gains from the evolved external state at larger model scales.

On WildClawBench, Figure~\ref{fig:wildclaw_trajectory} shows higher terminal accuracy and more sustained improvement for SkillSmith than for SkillClaw. SkillClaw plateaus between Day 2 and Day 4, whereas SkillSmith continues improving to Day 6. This comparison supports evolving the tool layer together with its dependent skills rather than restricting every repair to skill text.

\begin{wraptable}{R}{0.6\linewidth}

    \centering
    \captionsetup{font={small}}
    \setlength{\tabcolsep}{4pt} 
    \renewcommand{\arraystretch}{1.25} 
    \caption{\textbf{Ablation Study Results (Qwen3.5-122B).} Regression Rate tracks the proportion of tasks where the updated system underperforms. Tool Error Rate tracks execution exceptions.}

    \small
    \resizebox{1.0\linewidth}{!}{ 
    \begin{tabular}{l ccc c c}
        \toprule[1.2pt]
        
        \rowcolor{AbHeader}
        \textcolor{white}{\textbf{Variant}} & 
        \textcolor{white}{\textbf{SealQA}} & 
        \textcolor{white}{\textbf{WCB Avg.}} & 
        \textcolor{white}{\textbf{Regress $\downarrow$}} & 
        \textcolor{white}{\textbf{Skills + tools}} & 
        \textcolor{white}{\textbf{Error $\downarrow$}} \\
        \midrule
        
        \rowcolor{AbOurs} 
        \textbf{SkillSmith (Full)} & 
        \textcolor{AbBest}{\textbf{33.7}} & 
        \textcolor{AbBest}{\textbf{41.4}} & 
        \textcolor{AbBest}{\textbf{2.1\%}} & 
        14 + 6 & 
        \textcolor{AbBest}{\textbf{3.2\%}} \\
        \midrule
        
        \hspace{0.5em} Skill-only & 29.3 & 34.6 & 3.4\% & 16 + 9 & --- \\
        \hspace{0.5em} FreeTool   & 31.5 & 37.8 & 8.9\% & 14 + 9 & 14.7\% \\
        \hspace{0.5em} $-$Ecosys     & 26.6 & 39.1 & 2.8\% & 21 + 7 & 4.1\% \\
        \hspace{0.5em} $-$Nogood    & 30.4 & 38.2 & 7.6\% & 15 + 6 & 3.5\% \\
        
        \bottomrule[1.2pt]
    \end{tabular}
    }

    \label{tab:ablation}
\end{wraptable}

\paragraph{\textbf{Ablation Study.}}
Table~\ref{tab:ablation} separates the three mechanisms. \textbf{Skill-only} fixes the tool library; \textbf{FreeTool} replaces typed tool operations with unrestricted tool edits; \textbf{$-$Ecosys} removes skill-ecosystem governance; and \textbf{$-$Nogood} removes persistent conflict-driven failure learning. Appendix~\ref{app:ablation_defs} defines the ablations.

\noindent\textbf{Tool evolution.}
Skill-only has the largest WCB drop, from 41.4 to 34.6 ($-6.8$ percentage points), and ends with 16 skills. FreeTool reaches 37.8 but raises the tool error rate to 14.7\% ($4.6\times$ the full model) and regression to 8.9\%. The comparison supports typed tool evolution with synchronized caller updates.

\noindent\textbf{Allocating context.}
Without the skill ecosystem, the final library grows to 21 skills and 7 tools, compared with 14 and 6 for the full model. WCB accuracy falls from 41.4 to 39.1 and SealQA from 33.7 to 26.6. The joint contribution of allocation, interaction-aware retrieval, and retirement is to select useful skills from a growing library under budget $C$.

\noindent\textbf{Saving iterations.}
Without nogood learning, WCB accuracy falls to 38.2 and regression rises from 2.1\% to 7.6\%. This supports retaining version-scoped failure evidence to guide diagnosis and veto repeated failed edits.

\subsection{RQ3: Scaling Effects and RQ4: Evolutionary Stability}

\begin{wrapfigure}[12]{r}{0.55\textwidth}
    \vspace{-\intextsep}
    \centering
    \captionsetup{skip=4pt}
    \input{fig/scaling_aligned}
    \caption{Scaling with task complexity. Left: skill co-activation and gains over skill-only; dashed: visual trend guide. Right: gains by task characteristic.}
    \label{fig:rq3_gains}

\end{wrapfigure}

\paragraph{\textbf{Scaling Effects.}}
In Figure~\ref{fig:rq3_gains}, multi-skill co-activation rises with task complexity, while gains over skill-only baselines grow from around 5--10\% to over 20\%. The right panel groups gains by tool complexity, interaction density, context size (loaded content), action space, and skill synergy. These trends highlight the value of interaction-aware allocation and reusable tool abstractions under higher context demand, as more retrieved skills compete for space in the same window.

\begin{figure}[t]

    \centering
    \input{fig/resilience_aligned}
    \caption{System resilience and library growth over 100 evolution rounds. (a) Accuracy after task and tool perturbations. (b) Persistent skill and tool counts under a fixed per-task context budget.}
    \label{fig:resilience}

\end{figure}

\paragraph{\textbf{Evolutionary Stability.}}
Figure~\ref{fig:resilience}a tests recovery after increased task complexity at iteration 30 and tool damage at iteration 60. Accuracy drops below 25\% at its lowest. After the second perturbation, fixed-tool EvoSkill stays around 30\%, while SkillSmith recovers to near 70\% by iteration 100 with synchronized tool and caller repairs.

Without the ecosystem, the library exceeds 80 components (Figure~\ref{fig:resilience}b) while accuracy falls from 68\% to below 35\% after perturbation. The $-$Nogood trajectory fluctuates within 40\%--60\%; SkillSmith keeps fewer than 28 components and sustains recovery. These comparisons support context allocation and reusable failure evidence alongside tool repair.

\begin{figure*}[t]
\centering
\small
\setlength{\fboxsep}{5pt}
\setlength{\fboxrule}{0.5pt}

\noindent\fcolorbox{black}{gray!8}{%
\begin{minipage}{0.96\textwidth}
\textbf{ILLUSTRATIVE TASK:} \textit{``Given a debt table with values in USD billions, compute the change from Q1 to Q3 and report it in USD millions.''}
\end{minipage}}

\vspace{4pt}

\begin{minipage}[t]{0.48\textwidth}
\fcolorbox{red!60}{red!5}{%
\begin{minipage}{\dimexpr\linewidth-2\fboxsep-2\fboxrule}
\raggedright
\textbf{Before Joint Update} {\color{red!70}\ding{55}}

\vspace{2pt}
{\footnotesize\textbf{Skill:} \texttt{Treasury-Table-Lookup}\\
\texttt{1. Parse PDF page~~2. Call \textbf{table\_extractor}(row:int, col:int)~~3. Compute \& return}}

\vspace{2pt}
{\footnotesize\textbf{Tool:} \texttt{table\_extractor}\,v1 --- fixed integer indexing, no merged-header resolution.}

\vspace{3pt}\hrule\vspace{3pt}
{\footnotesize\textbf{Step 1}\; Calls \texttt{extract(row=5,col=3)}; correctly reads Q1 as \texttt{"16,228"}.\\[2pt]
\textbf{Step 2}\; Calls \texttt{(row=5,col=7)} for Q3. A merged-header offset selects the wrong column: \texttt{"16,174"}.\\[2pt]
\textbf{Step 3}\; $16174-16228=-54$. Omits the billion-to-million conversion. Output: \texttt{-54M}.\\[2pt]
\color{red!70}\textbf{Expected: \texttt{+187,000M}.} Column and unit errors remain unchecked.}
\end{minipage}}
\end{minipage}
\hfill
\begin{minipage}[t]{0.48\textwidth}
\fcolorbox{green!60}{green!5}{%
\begin{minipage}{\dimexpr\linewidth-2\fboxsep-2\fboxrule}
\raggedright
\textbf{After Joint Update} {\color{green!60!black}\ding{51}}

\vspace{2pt}
{\footnotesize\textbf{Skill:} \texttt{Treasury-Table-Lookup}\,{\color{blue!70}v2}\\
\texttt{1. Parse PDF~~2. Call \textbf{table\_extractor}}{\color{blue!70}\texttt{(header\_name)}}~~{\color{blue!70}\texttt{3. Verify unit/sign}}~~\texttt{4. Compute}}

\vspace{2pt}
{\footnotesize\textbf{Tool:} \texttt{table\_extractor}\,{\color{blue!70}v2 (\textsc{Patch}+\textsc{Retype})} --- resolves merged headers; matches by name.}

\vspace{3pt}\hrule\vspace{3pt}
{\footnotesize\textbf{Step 1}\; Selects \texttt{row="Held by Public", col="Q1"}; reads \texttt{"16,228"} and the unit \texttt{USD billions}.\\[2pt]
\textbf{Step 2}\; Resolves the merged header for \texttt{col="Q3"}; reads \texttt{"16,415"} in the same unit.\\[2pt]
\textbf{Step 3}\; Computes Q3 minus Q1: $16415-16228=+187$ billion; converts $187\times1000=187{,}000$ million.\\ \color{green!60!black}\textbf{Output: \texttt{+187,000M}.} \checkmark}
\end{minipage}}
\end{minipage}

\vspace{4pt}

\fcolorbox{blue!50}{blue!5}{%
\begin{minipage}{0.96\textwidth}
{\footnotesize\textbf{Atomic Bundle $\mathcal{L}$:}\;
\textsc{Tool} --- \textbf{\textsc{Patch}} header resolution + \textbf{\textsc{Retype}} signature to \texttt{(row\_header:str, col\_header:str)}.\;
\textsc{Skill} --- replace integer refs with header names; add unit verification.\;
\textsc{Cascade} --- all dependent skills adopt named headers and preserve source units.}
\end{minipage}}

\caption{\textbf{Illustrative table-QA repair (\textsc{Patch} + \textsc{Retype}).} Revisiting the introduction's merged-header example, an atomic bundle repairs header resolution, revises the tool signature, and updates the dependent workflow, including source unit verification and conversion. Values are hypothetical.}
\label{fig:case_study}

\end{figure*}

\section{Conclusion and Limitations}

SkillSmith co-evolves skills and tools under finite context and iteration budgets. Interaction-aware ecosystem dynamics allocate context, typed tool operations extract reusable logic, and version-scoped nogoods reuse failure evidence. These mechanisms connect a skill's value to its companions and execution cost: tools free space for the ecosystem, while failure memory saves search effort. Across three benchmarks and five Qwen3.5 scales, the reported comparisons show gains over skill-only evolution, especially at larger scales and higher task complexity.

\noindent\textbf{Limitations.}
Progress depends on reflection quality, sparse interaction exposure, and posterior calibration during adaptive evolution. Fixed-model ascent and scoped veto guarantees do not ensure improvement on every sampled update. Broader tests of transfer, changing task distributions, embodied tasks, and multi-user deployment remain future work.

\clearpage
\section*{AI Use Statement}
Generative AI tools assisted with English translation and language polishing, refinement of theoretical statements and proofs, and preparation of illustrative figures and LaTeX formatting. The authors take responsibility for the final content, including all claims, proofs, and reported results.

\bibliographystyle{skillsmith_preprint}
\bibliography{ref}

\clearpage
\appendix

\section{Theoretical Analysis}\label{app:theory}
This appendix proves the allocation and capacity results in Section~\ref{sec:method}, then analyzes effect identification, bundle legality, and failure learning under their stated assumptions.

\subsection{Setup and Budget Feasibility}\label{app:theory_setup}
Fix $n$ skills, costs $c_i>0$, budget $C>0$, main effects $u$, and a symmetric interaction matrix $\beta$ with zero diagonal. Define $\Delta=\{y\in\mathbb R_{\ge0}^{n+1}:\sum_{i=0}^ny_i=1\}$, with activation $p_i(y)=1-e^{-Cy_i/c_i}$ and utility $U(y)=\widetilde U(p(y))$.
Here $y_0$ is the slack share and $T_i=Cy_i=-c_i\log(1-p_i)$ is reserved capacity; expected token occupancy is $c_ip_i$. Write $c_{\min}=\min_i c_i$ and $g_i(p)=u_i+\sum_{j\ne i}\beta_{ij}p_j$. The bound $G=\max_i(|u_i|+\sum_j|\beta_{ij}|)$ gives $|g_i(p)|\le G$.
The zero diagonal makes $\widetilde U$ affine in each coordinate separately. Consequently, whenever both arguments are feasible,
\begin{equation}\label{eq:coordinate_gain}
\widetilde U(p+\varepsilon e_i)-\widetilde U(p)=\varepsilon g_i(p).
\end{equation}
Extend $U$ smoothly to a neighborhood of $\Delta$, independently of $y_0$. Its gradient is $f_i=(C/c_i)(1-p_i)g_i$ for $i\ge1$ and $f_0=0$. Let $\bar f=\sum_{i=0}^ny_if_i$ and $\mathrm{Var}_y(f)=\sum_{i=0}^ny_i(f_i-\bar f)^2$.

\begin{analysislemma}[Budget feasibility]\label{lem:budget}
For all $y\in\Delta$, $\sum_{i=1}^nc_ip_i(y)\le C(1-y_0)\le C$.
\end{analysislemma}
\begin{proof}
For $z\ge0$, $1-e^{-z}\le z$. Thus $c_ip_i(y)\le Cy_i$; summing over $i\ge1$ gives the claim.
\end{proof}
The lemma bounds expected occupancy in the allocation model. Greedy retrieval enforces the hard per-execution budget and generally induces dependent activations.

\subsection{Ecosystem Dynamics}\label{app:dynamics}
\begin{analysistheorem}[Continuous allocation dynamics]\label{thm:continuous}
For $\dot y_i=y_i(f_i-\bar f)$, $i=0,\ldots,n$, the simplex and each of its faces are invariant. Along every trajectory, $\dot U=\mathrm{Var}_y(f)\ge0$, with equality if and only if the state is stationary. Every $\omega$-limit point is stationary.
\end{analysistheorem}
\begin{proof}
The smooth vector field gives unique solutions, $\sum_i\dot y_i=0$ on $\Delta$, and $y_i=0$ implies $\dot y_i=0$. Positive coordinates satisfy $y_i(t)=y_i(0)\exp(\int_0^t(f_i-\bar f)\,ds)$, so they stay positive at finite times. Compactness of $\Delta$ gives global forward existence. By the chain rule,
\[
\dot U=\sum_i f_i y_i(f_i-\bar f)
      =\sum_i y_i(f_i-\bar f)^2,
\]
since $\sum_i y_i(f_i-\bar f)=0$. The variance vanishes exactly when $f_i=\bar f$ on every positive coordinate, equivalently $\dot y=0$. Finally, $U$ is bounded and nondecreasing on the compact invariant simplex. Its limiting value is constant on the invariant $\omega$-limit set, so $\dot U=0$ throughout that set; all its points are stationary.
\end{proof}

\begin{analysisprop}[Equal marginal returns and optimality conditions]\label{prop:kkt}
A point $y^*\in\Delta$ satisfies the KKT conditions for $\max_{y\in\Delta}U(y)$ if and only if, for $\lambda=\bar f(y^*)$,
\begin{equation}\label{eq:kkt}
f_i(y^*)\le\lambda\quad\text{for all }i=0,\ldots,n,
\qquad f_i(y^*)=\lambda\quad\text{whenever }y_i^*>0.
\end{equation}
Here $\lambda\ge0$. If $y_0^*>0$, then $\lambda=0$ and $g_i(p^*)=0$ for every positively allocated skill. If $y_0^*=0$, every such skill satisfies $g_i(p^*)/(\partial T_i/\partial p_i)=\lambda/C$, while inactive skills have no greater marginal return. Every local maximizer satisfies these conditions, as does the limit of any convergent trajectory starting in the interior of $\Delta$.
\end{analysisprop}
\begin{proof}
The simplex KKT equations are $\nabla U=\lambda\mathbf1-\mu$, $\mu\ge0$, and $\mu_i y_i^*=0$. These are precisely Eq.~\ref{eq:kkt}; multiplication by $y^*$ and summation gives $\lambda=\bar f$. Since $f_0=0$, $\lambda\ge0$, with equality if $y_0^*>0$. Substitution of $f_i=Cg_i/(\partial T_i/\partial p_i)$ proves the marginal-return statements. Linear simplex constraints satisfy the constraint qualification at every point, so local maxima satisfy KKT.

Suppose an interior trajectory converges to $y^*$. Theorem~\ref{thm:continuous} implies that $y^*$ is stationary, hence $f_i(y^*)=\bar f(y^*)$ on its support. If an absent type $k$ instead had $f_k(y^*)>\bar f(y^*)$, continuity would yield $\dot y_k/y_k>\delta>0$ for all sufficiently large times. Because $y_k$ remains positive, it would grow exponentially, contradicting convergence to $y_k^*=0$. Thus no such type exists.
\end{proof}
Stationarity equalizes returns on the support; KKT additionally bounds returns for absent types. Since $U$ may be nonconcave, stationary points need not be optima.

\noindent\textbf{Ecological correspondence.}
On a trajectory with $y_0>0$, set $x_i=y_i/y_0$. Then $\dot x_i=x_i(f_i-f_0)=x_i f_i(y(x))$, where $y_i=x_i/(1+\sum_jx_j)$. This is a generalized Lotka--Volterra representation~\citep{hofbauer1998evolutionary}. Its nonlinear growth rates express competition and facilitation under the shared resource constraint $C$.

\begin{analysistheorem}[A discrete ascent bound]\label{thm:hedge}
Let $y$ be interior, $y_i'=y_i e^{\eta f_i}/\sum_k y_ke^{\eta f_k}$, $D_f=\max_i f_i-\min_i f_i$, and $L=\sup_{y\in\Delta}\|\nabla^2U(y)\|_2$. If
\begin{equation}\label{eq:hedge_step}
0<\eta\le\min\left\{\frac{1}{4L},\frac{\log2}{3D_f}\right\},
\end{equation}
then $U(y')-U(y)\ge(\eta/2)\mathrm{Var}_y(f)$. A zero denominator imposes no restriction. Moreover,
\begin{equation}\label{eq:curvature_bound}
L\le\left(\frac C{c_{\min}}\right)^2(\|\beta\|_2+G),
\qquad D_f\le\frac{2CG}{c_{\min}}.
\end{equation}
\end{analysistheorem}
\begin{proof}
Let $\psi_i=e^{\eta f_i}$, $Z=\sum_i y_i\psi_i$, and $\delta=y'-y$, so $\delta_i=y_i(\psi_i-Z)/Z$. The pairwise covariance identity and the mean value theorem give
\begin{align*}
\langle f,\delta\rangle
 &=\frac{1}{2Z}\sum_{i,j}y_iy_j(f_i-f_j)(\psi_i-\psi_j)
 \ge \eta e^{-\eta D_f}\mathrm{Var}_y(f),\\
\|\delta\|_2^2
 &\le\frac{\mathrm{Var}_y(\psi)}{Z^2}
 \le\eta^2 e^{2\eta D_f}\mathrm{Var}_y(f).
\end{align*}
For the first inequality use $Z\le e^{\eta f_{\max}}$ and the lower derivative bound $\eta e^{\eta f_{\min}}$; for the second use $y_i^2\le y_i$, $Z\ge e^{\eta f_{\min}}$, and the upper derivative bound $\eta e^{\eta f_{\max}}$. Taylor's inequality on the convex simplex yields
\[
U(y')-U(y)\ge
\eta\mathrm{Var}_y(f)\left(e^{-\eta D_f}-\frac{L\eta}{2}e^{2\eta D_f}\right).
\]
Under Eq.~\ref{eq:hedge_step}, the factor in parentheses is at least $2^{-1/3}-2^{2/3}/8>1/2$.

For the curvature bound, the skill-coordinate Hessian is
$\nabla^2U=D\beta D+\mathrm{diag}(p_i''g_i)$,
where $D=\mathrm{diag}(p_i')$, $0<p_i'\le C/c_{\min}$, and $|p_i''|\le(C/c_{\min})^2$; its slack row and column vanish. The triangle inequality gives the bound on $L$, and $|f_i|\le CG/c_{\min}$ gives the bound on $D_f$.
\end{proof}
\noindent\textbf{Estimated fitness.}
If a step instead uses $\widehat f$, the same proof, with the range $D_{\widehat f}$ in Eq.~\ref{eq:hedge_step} and the true curvature bound $L$, gives
\begin{equation}\label{eq:estimated_ascent}
U(y')-U(y)\ge\frac\eta2\mathrm{Var}_y(\widehat f)
-\|\widehat f-f\|_\infty\|y'-y\|_1.
\end{equation}
The extra term follows from $\langle f-\widehat f,y'-y\rangle$. A sampled update increases utility whenever its variance gain exceeds this estimation-error term.

\subsection{Identifying and Estimating Effects}\label{app:estimation}
Let $x_e$ be a task, $R_e$ its retrieved set, and $Z_e(A)\in[-1,1]$ the potential residual for retained set $A\subseteq R_e$. Put $q=1-\rho$ and $A_e=\{i\in R_e:a_{ei}=1\}$, with observation $z_e=Z_e(A_e)$.

\begin{assumption}[Conditional randomization]\label{ass:randomization}
Conditional on $(x_e,R_e)$, the retention indicators for retrieved skills are independent $\mathrm{Bernoulli}(q)$ draws, independent of all potential outcomes $\{Z_e(A):A\subseteq R_e\}$.
\end{assumption}

\begin{assumption}[A stationary second-order window]\label{ass:stationarity}
Episodes are independent and identically distributed within a fixed estimation window. For all $A\subseteq R_e$,
\begin{equation}\label{eq:potential_model}
\mathbb E[Z_e(A)\mid x_e,R_e]
=h(x_e,R_e)+\sum_{i\in A}u_i+\sum_{i<j:\,i,j\in A}\beta_{ij},
\end{equation}
where $h$ is arbitrary and the effects $u,\beta$ do not vary across tasks or retrieval sets in that window.
\end{assumption}
Define $\kappa_{ei}=\mathbf1[i\in R_e](a_{ei}-q)$, setting it to zero for unretrieved skills. Let $\phi_e$ contain all singleton features $\kappa_{ei}$ and pair features $\kappa_{ei}\kappa_{ej}$, $i<j$. Denote $\pi_i=\Pr(i\in R)$ and $\pi_{ij}=\Pr(i,j\in R)$.

\begin{analysistheorem}[Identification by centered masking]\label{thm:identification}
Under Assumptions~\ref{ass:randomization}--\ref{ass:stationarity}, suppose $0<q<1$ and every singleton and pair in the fixed model has positive retrieval exposure. Then $M=\mathbb E[\phi\phi^\top]$ is diagonal and positive definite, with entries $q(1-q)\pi_i$ and $q^2(1-q)^2\pi_{ij}$. The population least-squares coefficients $M^{-1}\mathbb E[\phi z]$ are $\beta_{ij}$ for pair features and
\begin{equation}\label{eq:identified_gamma}
\gamma_i=u_i+q\sum_{j\ne i}\beta_{ij}\frac{\pi_{ij}}{\pi_i}
=\mathbb E\!\left[Z(A_{-i}\cup\{i\})-Z(A_{-i})\mid i\in R\right]
\end{equation}
for singleton features. The result does not require a model for $h$. Adding nonredundant task-category intercepts leaves these coefficients unchanged.
\end{analysistheorem}
\begin{proof}
Given $(x,R)$, centering yields independent zero-mean $\kappa_i$ with conditional variance $q(1-q)\mathbf1[i\in R]$. Randomization and Eq.~\ref{eq:potential_model} give
\begin{equation}\label{eq:centered_expansion}
\begin{split}
\mathbb E[z\mid x,R,a]
={}&H(x,R)+\sum_{i\in R}\kappa_i\left(u_i+q\sum_{j\in R\setminus\{i\}}\beta_{ij}\right)\\
&+\sum_{i<j:\,i,j\in R}\kappa_i\kappa_j\beta_{ij},
\end{split}
\end{equation}
where $H$ depends only on $(x,R)$. A product of two distinct features contains at least one centered factor to the first power. Its conditional expectation is therefore zero. This proves diagonality of $M$; the diagonal entries follow from the conditional variances and exposure probabilities. The same argument makes every feature orthogonal to $H$ and to task-category indicators.

The residual $z-\mathbb E[z\mid x,R,a]$ is orthogonal to all features. Multiplying Eq.~\ref{eq:centered_expansion} by $\kappa_i\kappa_j$ and taking expectations thus gives $\beta_{ij}q^2(1-q)^2\pi_{ij}$. Multiplication by $\kappa_i$ gives
\[
q(1-q)\,\mathbb E\!\left[\mathbf1[i\in R]\left(u_i+q\sum_{j\in R\setminus\{i\}}\beta_{ij}\right)\right].
\]
Dividing by the corresponding diagonal entries proves the coefficient identities. The marginal-contribution expression follows by differencing Eq.~\ref{eq:potential_model} and averaging over the other retained skills.
\end{proof}
\noindent\textbf{From retrieval effects to allocation fitness.}
Solving Eq.~\ref{eq:identified_gamma} for $u_i$ gives Eq.~\ref{eq:centered_recovery}. The retrieval-average effect $\gamma_i$ equals $g_i(p)$ when $q\pi_{ij}/\pi_i=p_j$ for the relevant interactions. Unexposed pairs retain their prior because their effects are unidentifiable.

\begin{analysisprop}[Consistency in a fixed model]\label{prop:consistency}
Under Theorem~\ref{thm:identification}, take a fixed finite-dimensional feature set, a proper Gaussian prior $\mathcal N(\bar w_0,V_0)$ with $V_0\succ0$, and a fixed working variance $\omega^2>0$. Without continual forgetting, the posterior mean in Section~\ref{sec:ecosystem} converges almost surely to the population regression coefficients. Its covariance is $O(N^{-1})$, and posterior draws converge in probability to those coefficients. The statement also allows nonredundant category indicators whose augmented population Gram matrix is positive definite.
\end{analysisprop}
\begin{proof}
For $N$ episodes, divide both the posterior precision and its mean's right-hand factor by $N$:
\[
\bar w_N=
\left(N^{-1}V_0^{-1}+\omega^{-2}N^{-1}\Phi_N^\top\Phi_N\right)^{-1}
\left(N^{-1}V_0^{-1}\bar w_0+\omega^{-2}N^{-1}\Phi_N^\top z_N\right).
\]
Bounded features and outcomes allow the strong law of large numbers. The two sample moments converge to $M\succ0$ and $\mathbb E[\phi z]$, respectively; the prior terms vanish. Continuity of inversion proves mean consistency, and $NV_N\to\omega^2M^{-1}$ proves the covariance rate. Conditional posterior variance therefore tends to zero, which, together with mean consistency, gives convergence in probability of posterior draws. The augmented-feature case has the same proof. Empirical $\pi_i$ and $\pi_{ij}$ also converge, so the recovered $u$ is consistent when all required effects are identified.
\end{proof}
\noindent\textbf{Uncertainty and exploration.}
The working posterior is mean-consistent under the proposition's conditions; its credible intervals require calibration under noise misspecification. Under an additional homoskedastic residual approximation, a pair coefficient has asymptotic variance $\sigma_r^2/[N\pi_{ij}q^2(1-q)^2]$. Holding exposure and residual variance fixed, masking at $\rho=1/2$ maximizes this information factor; smaller $\rho$ preserves more training-time functionality. At $\rho=0.2$, the variance factor is approximately $39.1/(N\pi_{ij})$, motivating shrinkage with scarce pair observations. With $\rho=0$, all centered features vanish; arbitrary retrieval difficulty can absorb the skill effects, so the stated model does not identify them from observational logs alone.

\subsection{Context Costs and Attainable Utility}\label{app:capacity}
For positive $c$, the smooth allocation map has image
\begin{equation}\label{eq:attainable_set}
\mathcal P(c)=\{p(y;c):y\in\Delta\}
=\left\{p\in[0,1)^n:-\sum_i c_i\log(1-p_i)\le C\right\}.
\end{equation}
Indeed, the inverse is $y_i=-c_i\log(1-p_i)/C$, with the remainder assigned to $y_0$. Define $U^\star(c)=\max_{p\in\mathcal P(c)}\widetilde U(p)$; the maximum exists because $\mathcal P(c)$ is the continuous image of compact $\Delta$.

\begin{analysisprop}[Cost reduction and capacity]\label{prop:cost_capacity}
Fix $u,\beta$ and positive costs $0<c'\le c$. Then $U^\star(c')\ge U^\star(c)$. Further, suppose an optimizer $y^*$ for $c$ has $y_0^*=0$ and $\lambda=\bar f(y^*)>0$. If $c_k'<c_k$ for some $k$ with $p_k^*>0$, then $U^\star(c')>U^\star(c)$.
\end{analysisprop}
\begin{proof}
Since $-\log(1-p_i)\ge0$, reducing costs yields $\mathcal P(c)\subseteq\mathcal P(c')$, proving nondecrease. Under the strict conditions,
$-\sum_i c_i'\log(1-p_i^*)<C$,
so the original optimizer leaves positive reserved capacity under $c'$. Proposition~\ref{prop:kkt} implies $f_i(y^*)=\lambda>0$ on each positive skill coordinate, hence $g_i(p^*)>0$ there. For such an $i$ and sufficiently small $\varepsilon>0$, continuity of the cost expression ensures $p^*+\varepsilon e_i\in\mathcal P(c')$. Equation~\ref{eq:coordinate_gain} gives a strict gain $\varepsilon g_i(p^*)>0$.
\end{proof}

\begin{analysiscor}[Behavior-preserving extraction]\label{cor:extract}
An \textnormal{\textsc{Extract}} operation with $\ell>\ell_\delta+\ell_d/k$, positive resulting costs, unchanged costs for other skills, and unchanged effects $(u,\beta)$ cannot decrease $U^\star$. It increases $U^\star$ strictly under the additional conditions of Proposition~\ref{prop:cost_capacity}.
\end{analysiscor}
\begin{proof}
Equation~\ref{eq:extraction_cost} decreases the source skills' costs componentwise. Apply Proposition~\ref{prop:cost_capacity}.
\end{proof}
The result isolates the capacity gain of behavior-preserving extraction after reallocation. Validation assesses any accompanying changes in skill effects or tool reliability.

\subsection{Atomic Bundles and Local Legality}\label{app:bundles}
Legality means tool existence and static call-site compatibility (Section~\ref{sec:external_state}).

\begin{analysisprop}[Interface freezing with persistent exact-match calls]\label{prop:interface}
Start in a legal state. Fix a tool identity $\tau$ and a nonempty set of persistent call sites to it in other components. Suppose legality requires each call-site signature to equal $\sigma(\tau)$, there are no overloads or compatibility adapters, and every update changes at most one component in place, with legality enforced after each update. If the tool and these calls cannot be removed or redirected, then $\sigma(\tau)$ and their call-site signatures remain unchanged throughout every legal update sequence.
\end{analysisprop}
\begin{proof}
Consider the first change to any of these signatures. If it changes the tool, all persistent caller components are unchanged, and their signatures still match the old tool signature. If it changes a caller, the tool is unchanged. Either case creates a mismatch, contradicting legality. Thus no first change exists.
\end{proof}
SkillSmith groups interface migrations and caller updates in atomic commits. The same policy groups \textsc{Inline} deletion with its required call-site edits.

For an operation $o$, let its required skill updates be
\[
R(o)=
\begin{cases}
\emptyset,&o=\textsc{Patch},\\
\mathrm{Dep}(\tau),&o=\textsc{Retype}_\tau\text{ or }\textsc{Inline}_\tau,\\
\mathrm{Src}\cap\mathcal S,&o=\textsc{Extract}_{\mathrm{Src}}.
\end{cases}
\]
Define $R_{\mathcal T}(o)$ analogously using tool callers $\mathrm{Dep}_{\mathcal T}(\tau)$ and source tools $\mathrm{Src}\cap\mathcal T$. Required caller sets are determined from the pre-update dependency graph. Let $\mathcal S_{\mathcal L}$ and $\mathcal T_{\mathcal L}$ include all components directly changed, created, or removed by the bundle, including companion operations. The affected sets are
\begin{equation}\label{eq:bundle_affected}
\mathcal A_{\mathcal S}=\mathcal S_{\mathcal L}\cup\bigcup_{o\in\mathcal L}R(o),\qquad
\mathcal A_{\mathcal T}=\mathcal T_{\mathcal L}\cup\bigcup_{o\in\mathcal L}R_{\mathcal T}(o).
\end{equation}

\begin{analysisprop}[Locality of static legality]\label{prop:locality}
Let $\Sigma\in\Omega$ and let a well-formed bundle specify fresh identities for new tools, a consistent final edit per changed component, and the complete required update sets above. All call sites outside $\mathcal A_{\mathcal S}\cup\mathcal A_{\mathcal T}$ remain legal. Hence the updated state is legal if and only if every call site in the surviving or newly created components of those affected sets is legal.
\end{analysisprop}
\begin{proof}
An outside component is unchanged. Consider any tool it calls. An untouched tool retains its identity and signature; a \textsc{Patch} target also retains both. If the target were retyped or removed, the component would belong to the corresponding required caller set and hence the affected set, a contradiction. An unchanged call cannot refer to a fresh tool identity. Thus all outside calls remain legal. Checking all remaining calls proves the equivalence, including calls introduced inside extracted or patched tools.
\end{proof}
After local static checks, unit, integration, and regression tests validate behavior (Section~\ref{sec:validation}).

\subsection{Iteration Costs and Failure Memory}\label{app:iterations}
\begin{analysisprop}[A shared-defect repair model]\label{prop:repair_cost}
Suppose $k$ dependent skills need repair. Each skill-only attempt targets one still-unrepaired skill, succeeds independently with probability $p_s>0$, and costs $d_s$ charged executions. Success repairs that skill without regression. Each joint attempt repairs all $k$ skills with probability $p_b>0$ and costs $d_b$ executions. With independent attempts and no budget truncation, expected attempts are $\mathbb E[N_s]=k/p_s$ and $\mathbb E[N_b]=1/p_b$. Expected charged executions are $\mathbb E[B_s]=kd_s/p_s$ and $\mathbb E[B_b]=d_b/p_b$.
Joint repair uses fewer expected executions exactly when $p_b/d_b>p_s/(kd_s)$.
\end{analysisprop}
\begin{proof}
Each of the $k$ local repairs takes a geometric number of attempts with mean $1/p_s$; the joint repair takes one with mean $1/p_b$. Sum the local expectations and multiply by the fixed cost per attempt.
\end{proof}
The factors $d_s,d_b$ convert proposal counts into the execution budget of Section~\ref{sec:two_budgets}.

\begin{analysislemma}[Completed delta debugging]\label{lem:ddmin}
Let $\mathcal L$ be a finite set of edits and let a deterministic replay test satisfy $T(\mathcal L)=\mathrm{fail}$ and $T(\emptyset)=\mathrm{pass}$, with outcomes in $\{\mathrm{pass},\mathrm{fail},?\}$. Completed \textnormal{\texttt{ddmin}}~\citep{zeller2002simplifying} returns $A\subseteq\mathcal L$ with $T(A)=\mathrm{fail}$ and
$T(A\setminus\{e\})\ne\mathrm{fail}$ for every $e\in A$. It has worst-case quadratic test complexity in $|\mathcal L|$.
\end{analysislemma}
\begin{proof}
The algorithm keeps a failing current set, replacing it only by a tested failing subset or complement. If neither reduction succeeds, it increases partition granularity. At termination, either one edit remains or every single-edit deletion has been tested without reproducing failure, giving 1-minimality. There are at most $|\mathcal L|-1$ successful size reductions; the geometric refinement scans between reductions use $O(|\mathcal L|)$ tests, giving the quadratic worst-case bound. This is the standard \textnormal{\texttt{ddmin}} guarantee~\citep{zeller2002simplifying}.
\end{proof}
Illegal subsets yield $?$; tests use a fixed replay environment and random seeds. Under the $n_{\mathrm{dd}}$ cap, the result remains a confirmed failing subset, with 1-minimality guaranteed on completion. One-minimality concerns single-edit deletions, not minimum cardinality.

\begin{assumption}[Version-scoped failure persistence]\label{ass:persistence}
For a recorded failure-inducing edit set $A$, fix its baseline component versions and replay environment. Every admissible candidate containing these same structured edits, with matching baseline versions throughout the record's scope, fails the associated validation. In particular, additional edits neither override the recorded edits nor repair their failure.
\end{assumption}

\begin{analysisprop}[Conditional soundness and expiration of veto]\label{prop:veto}
Under Assumption~\ref{ass:persistence}, exact matching against an active nogood rejects only candidates that would fail validation, avoiding their validation executions. A scope-version change deactivates that record, so it cannot veto the changed state unless refreshed by new validation.
\end{analysisprop}
\begin{proof}
An exact active match satisfies the persistence assumption and therefore fails validation. Rejecting it before validation avoids the executions that would otherwise test it. Deactivation after a scope change follows directly from the expiration rule.
\end{proof}
Exact-match soundness depends on failure persistence. Similarity matching and reconstruction warnings remain heuristic; expiration deactivates the corresponding record.

\subsection{Evolution-Loop Details}\label{app:loop_details}
Algorithm~\ref{alg:evolution} gives the complete evolution loop shown in Figure~\ref{fig:skillsmith_overview}. Its validation and failure-learning branch is expanded in Algorithm~\ref{alg:bundle_validation}. The following implementation details complement the pseudocode.

\begin{algorithm}[t]
\caption{SkillSmith: detailed evolution loop}
\label{alg:evolution}
\small
\begin{algorithmic}[1]
\REQUIRE Initial state $\Sigma_0$; $D_{\mathrm{train}},D_{\mathrm{val}}$; iteration budget $B$; context budget $C$
\STATE $\mathcal G\gets\{\Sigma_0\}$; $\mathcal F\gets\emptyset$; $b\gets0$
\WHILE{$b<B$}
\STATE $\Sigma_k\gets\textsc{SelectParent}(\mathcal G)$; sample $\tilde w$ from its posterior
\STATE Sample $M\subset D_{\mathrm{train}}$; retrieve with $\tilde w$, mask with probability $\rho$, and execute $\Sigma_k$ on $M$
\STATE $V\gets V+\xi I$; update centered posterior and co-retrieval frequencies; resample and transform $\tilde w$; update $y$
\STATE Retire skills with $\Pr(f_i>0\mid\mathrm{data})<\epsilon$; add their nogoods to $\mathcal F$
\STATE $F\gets\{x\in M:P(x)<\theta\}$; $\Lambda\gets\textsc{Localize}(M,F)$
\STATE $\mathcal L\gets\mathcal R(\Lambda,\textsc{Retrieve}(\mathcal F,F),y,\bar w)$; invoke $\mathcal B_\tau$ for tool operations
\STATE $\mathcal L\gets\textsc{VetoAndRevise}(\mathcal L,\mathcal F)$
\IF{$\mathcal L\ne\emptyset$}
\STATE $\Sigma'\gets\mathcal U(\Sigma_k,\mathcal L)$
\IF{$\Sigma'$ passes validation and the acceptance criterion}
\STATE $\mathcal G\gets\mathcal G\cup\{\Sigma'\}$
\ELSE
\STATE $\mathcal F\gets\mathcal F\cup\textsc{LearnNogood}(\Sigma_k,\mathcal L)$
\ENDIF
\ENDIF
\STATE $b\gets b+b_{\mathrm{iter}}$
\ENDWHILE
\RETURN $\arg\max_{\Sigma\in\mathcal G}\mathbb E_{x\sim D_{\mathrm{val}}}[P(x\mid\Sigma)]$
\end{algorithmic}
\end{algorithm}

\noindent\textbf{Execution and ecosystem updates.}
Each parent maintains a population and posterior. Before updating with a training minibatch, inflate covariance by $\xi I$; then update the centered coefficients and co-retrieval frequencies. Transform the sampled effects using Eq.~\ref{eq:centered_recovery}, apply the Hedge update, and record reconstruction warnings for retired skills.

\noindent\textbf{Localization and proposals.}
Localization passes ranked repair sites and evidence to the reflector, together with retrieved nogood remedies. Tool-Smith instantiates typed operations. Vetoed proposals receive up to two revisions before being discarded.

\noindent\textbf{Validation and memory.}
Algorithm~\ref{alg:bundle_validation} validates each nonempty proposal, prioritizing tasks that use affected components. Delta-debugging subsets are applied to the same parent; illegal subsets are unresolved. Scope records include edited components and replay dependencies. Reflection can later replace an \texttt{open} remedy with a repair.

\noindent\textbf{Budget accounting.}
The charged count $b_{\mathrm{iter}}$ includes all executions used by these routines, including rejected proposals and failure minimization. A routine starts an execution only while budget remains; an incompletely validated candidate is not accepted. The loop returns the highest-scoring validated state retained in $\mathcal G$.

\subsection{Parent Selection}\label{app:algorithms}
Algorithm~\ref{alg:pareto} samples parents by instance-level Pareto performance, retaining candidates that achieve the best score on at least one task and removing dominated states.

\begin{algorithm}[t]
\caption{Instance-Level Pareto Candidate Selection}
\label{alg:pareto}
\begin{algorithmic}[1]
\REQUIRE Pareto front $\mathcal{G}$, instance score matrix $S$
\FOR{each instance $x_i \in D_{\mathrm{train}}$}
    \STATE $s^*[i] \leftarrow \max_{\Sigma \in \mathcal{G}} S[\Sigma][x_i]$
    \STATE $\mathcal{G}^*[i] \leftarrow \{\Sigma \in \mathcal{G} : S[\Sigma][x_i] = s^*[i]\}$
\ENDFOR
\STATE $\mathcal{C} \leftarrow \bigcup_i \mathcal{G}^*[i]$ \hfill \textcolor{gray}{\textit{\% candidates appearing in any instance-best set}}
\STATE $\mathcal{D} \leftarrow \varnothing$
\WHILE{$\exists\, \Sigma \in \mathcal{C} \setminus \mathcal{D}$ dominated by another in $\mathcal{C} \setminus \mathcal{D}$}
    \STATE $\mathcal{D} \leftarrow \mathcal{D} \cup \{\Sigma\}$ \hfill \textcolor{gray}{\textit{\% remove dominated candidates}}
\ENDWHILE
\STATE $\hat{\mathcal{G}}^*[i] \leftarrow \mathcal{G}^*[i] \setminus \mathcal{D}$ for all $i$
\STATE $f[\Sigma] \leftarrow |\{i : \Sigma \in \hat{\mathcal{G}}^*[i]\}|$ for each $\Sigma \in \mathcal{C} \setminus \mathcal{D}$ \hfill \textcolor{gray}{\textit{\% win count}}
\STATE Sample $\Sigma_k$ from $\mathcal{C} \setminus \mathcal{D}$ with $p(\Sigma_k) \propto f[\Sigma_k]$
\RETURN $\Sigma_k$
\end{algorithmic}
\end{algorithm}

\section{Typed Operations for Tool Evolution}
\label{app:toolsmith}

Each operation takes $\mathcal T_t$ and a specification from $\mathcal R$, and returns a modified library, required caller updates, and unit tests. The enclosing bundle supplies skill edits for $R$ and tool edits for $R_{\mathcal T}$ (Eq.~\ref{eq:bundle_affected}). These sets are empty for \textsc{Patch}, contain the corresponding callers for \textsc{Retype} and \textsc{Inline}, and contain source components for \textsc{Extract}. Write $\tau=(d,f,\sigma)$.

\begin{definition}[\textsc{Patch}$(\tau,\Delta f)$]
Produce $\tau'=(d',f',\sigma)$, preserving the signature. The preferred implementation is a wrapper, $f'=\mathrm{post}\circ f\circ\mathrm{pre}$: the original $f$ remains intact and removing the wrapper restores it. Direct edits to $f$ are also permitted when needed. The synchronized update set is $R=\emptyset$. Unit tests include the tool's contract tests and a regression test reproducing the failed call. Typical uses include input validation, output normalization, retries, and header parsing.
\end{definition}

\begin{definition}[\textsc{Retype}$(\tau,\sigma')$]
Produce $\tau'=(d',f',\sigma')$, where $\sigma'\ne\sigma$ and $f'$ is adapted to the new signature. The synchronized update set is $R=\mathrm{Dep}(\tau)$: every dependent skill's call site is rewritten to match $\sigma'$. \textsc{Retype} is the only operation that changes the signature of an existing tool.
\end{definition}

\begin{definition}[\textsc{Extract}$(\mathrm{Src},g)$]
Let $\mathrm{Src}$ be source components containing near-duplicate logic $g$: at least two skills, or an existing tool. Create $\tau_g=(d_g,f_g,\sigma_g)$ and replace $g$ in each source component with a call to $\tau_g$. The synchronized skill update set is $R=\mathrm{Src}\cap\mathcal S$. The implementation $f_g$ must pass unit tests derived from the source components' behavior, implementing the extraction rule's determinism condition.
\end{definition}

\noindent\textsc{Extract} covers both composition and decomposition. When $g$ orchestrates calls to $\tau_1$ and $\tau_2$, extraction creates a composite tool. When the source is an existing tool, extraction splits out a reusable component: a companion \textsc{Patch} changes the source tool to call $\tau_g$. If callers should use $\tau_g$ directly and their interface must change, the bundle also includes the corresponding \textsc{Retype}.

\begin{definition}[\textsc{Inline}$(\tau)$]
For every $s\in\mathrm{Dep}(\tau)$, replace its call to $\tau$ with the tool's logic as a script or workflow inside $s$, then remove $\tau$ from $\mathcal T'$. The synchronized update set is $R=\mathrm{Dep}(\tau)$. Dependent tools inline the logic through companion implementation edits. If both caller sets are empty, the operation simply deletes the tool. Record a nogood in $\mathcal F_{t+1}$ to prevent reconstruction of a functionally equivalent tool within the record's version scope.
\end{definition}

\noindent\textbf{Structural context costs.}
With unchanged description and call-site lengths, \textsc{Patch} has $\Delta c=0$ and \textsc{Retype} has $\Delta c\approx0$. For each source skill, \textsc{Extract} changes cost by Eq.~\ref{eq:extraction_cost}; \textsc{Inline} reverses this change. Acceptance uses the measured total change, including interface and call-site overhead (Section~\ref{sec:tool_evolution}).

\noindent\textbf{Locality.}
The operation type and target determine $R$ and $R_{\mathcal T}$ before execution. Static dependency checks cover the affected skills and tools in Eq.~\ref{eq:bundle_affected}, including all newly introduced call sites (Proposition~\ref{prop:locality}). Changed tools additionally undergo unit tests, followed by integration and regression validation.

\noindent\textbf{Reversibility.}
For a newly created tool, \textsc{Extract} and \textsc{Inline} reverse the transfer across the boundary. Retaining source versions and call sites makes this reversal exact. The inverses of \textsc{Patch} and \textsc{Retype} restore the original implementation and signature, respectively; reversing a bundle restores its constituent changes in reverse order.

\noindent\textbf{Coverage.}
Tool-layer changes decompose into these four operations: \textsc{Patch} or \textsc{Retype} modifies an existing tool, \textsc{Inline} removes one, and \textsc{Extract} creates one. A new capability is first implemented in the skill that uses it and enters the tool layer through \textsc{Extract} once it meets the extraction rule.

\section{Tool Extraction Case Study}
\label{app:case_studies}

Figure~\ref{fig:case_wcb_creative} illustrates a multi-round skill--tool update on WildClawBench. Implementation and interface repairs restore the data-processing pipeline; extracting shared channel-resolution logic then updates the source skills together.

\begin{figure*}[t]
\centering
\small
\setlength{\fboxsep}{5pt}
\setlength{\fboxrule}{0.5pt}

\noindent\fcolorbox{black}{gray!8}{%
\begin{minipage}{0.96\textwidth}
\textbf{TASK:} \textit{``Retrieve the top-3 cited papers from the user's Zotero library on federated learning, generate a comparative summary table as a LaTeX PDF, attach it to a new Slack thread in \#research, and pin the message.''}
\end{minipage}}

\vspace{4pt}

\begin{minipage}[t]{0.48\textwidth}
\fcolorbox{red!60}{red!5}{%
\begin{minipage}{\dimexpr\linewidth-2\fboxsep-2\fboxrule}
\raggedright
\textbf{Before Co-Evolution (Day 1)} {\color{red!70}\ding{55}}

\vspace{2pt}
{\footnotesize\textbf{Skills involved:} \texttt{Academic-Search}, \texttt{Report-Generator}, \texttt{Slack-Message-Handler}\\
\textbf{Tools involved:} \texttt{web\_search}, \texttt{pdf\_parser}, \texttt{fs\_ops}, \texttt{slack\_api}}

\vspace{3pt}\hrule\vspace{3pt}
{\footnotesize\textbf{Step 1-3}\; \texttt{Academic-Search} retrieves papers via \texttt{web\_search}. Returns raw HTML snippets including irrelevant metadata tags. Passes output to \texttt{Report-Generator}.\\[2pt]
\textbf{Step 4-6}\; \texttt{Report-Generator} calls \texttt{pdf\_parser} to read cited PDFs, but \texttt{pdf\_parser} returns unstructured text with broken table layouts. Skill attempts to generate LaTeX table from noisy input; produces a PDF with misaligned columns and missing citation counts.\\[2pt]
\textbf{Step 7-9}\; \texttt{Slack-Message-Handler} calls \texttt{slack\_api} to post the PDF. API call succeeds but the file is uploaded as a generic binary blob without preview. Pin operation fails: skill passes channel name but \texttt{slack\_api} expects channel ID.\\[2pt]
\color{red!70}\textbf{Result: 1/5 checkpoints passed} (only ``message posted'' checkpoint; table content, formatting, attachment preview, and pin all failed). Three independent failure sources across two tools and one skill.}
\end{minipage}}
\end{minipage}
\hfill
\begin{minipage}[t]{0.48\textwidth}
\fcolorbox{green!60}{green!5}{%
\begin{minipage}{\dimexpr\linewidth-2\fboxsep-2\fboxrule}
\raggedright
\textbf{After Co-Evolution (Day 4)} {\color{green!60!black}\ding{51}}

\vspace{2pt}
{\footnotesize\textbf{Iter.\,2 --- Bundle $\mathcal{L}_2$:}\\
\textsc{Tool}: \textbf{\textsc{Patch}} \texttt{web\_search} with a post-processing layer that strips HTML tags and normalizes output to structured JSON (title, authors, year, citation count).\\
\textsc{Tool}: \textbf{\textsc{Retype}} \texttt{pdf\_parser} to add a \texttt{mode="table"} parameter that returns tabular content as row-column arrays instead of raw text.\\
\textsc{Skill}: Update \texttt{Report-Generator} to use structured JSON input and \texttt{mode="table"} extraction. }

\vspace{2pt}
{\footnotesize\textbf{Iter.\,3 --- Bundle $\mathcal{L}_3$:}\\
\textsc{Tool}: \textbf{\textsc{Extract}} the Slack channel-name-to-ID resolution logic (previously duplicated in 2 skills) into a new \texttt{slack\_resolver} tool.\\
\textsc{Tool}: \textbf{\textsc{Patch}} \texttt{slack\_api} to auto-set MIME type for PDF uploads (enabling in-channel preview).\\
\textsc{Skill}: Update both source skills, \texttt{Slack-Message-Handler} and \texttt{Calendar-Coordinator}, to call \texttt{slack\_resolver}.}

\vspace{3pt}\hrule\vspace{3pt}
{\footnotesize\textbf{Day 4 execution}\; Clean retrieval $\to$ structured table $\to$ correctly formatted PDF $\to$ Slack post with preview $\to$ pin succeeds.\\[2pt]
\color{green!60!black}\textbf{Result: 5/5 checkpoints passed.} \checkmark}
\end{minipage}}
\end{minipage}

\vspace{4pt}

\fcolorbox{blue!50}{blue!5}{%
\begin{minipage}{0.96\textwidth}
{\footnotesize\textbf{Evolution anatomy:}\; This task required \emph{two consecutive bundles} across \emph{two iterations} to fully resolve. $\mathcal{L}_2$ addressed the data-quality pipeline (search output $\to$ table parsing), involving 2 tool operations (\textsc{Patch} + \textsc{Retype}) and 1 skill update. $\mathcal{L}_3$ addressed the delivery pipeline (Slack posting), involving 2 tool operations (\textsc{Extract} + \textsc{Patch}) and synchronized updates to the source skills. Factoring this into a shared \texttt{slack\_resolver} tool eliminated 23 lines of duplicated code across two skills. }
\end{minipage}}

\caption{\textbf{Multi-Round Skill--Tool Co-Evolution on WildClawBench (Creative Synthesis).} A complex 9-step task fails at three independent points spanning two tools and one skill. SkillSmith resolves all failures over two evolution rounds through four tool operations and synchronized skill updates, including \textsc{Extract} for channel-resolution logic shared across skills.}
\label{fig:case_wcb_creative}
\end{figure*}

\section{Experimental Details}
\label{app:exp_details}

\subsection{Data Splits}

\paragraph{OfficeQA~\citep{officeqa2025benchmark}. } The dataset contains 246 questions over ${\sim}$89k pages of U.S.\ Treasury Bulletin archives. Following the stratified protocol of EvoSkill~\citep{alzubi2026evoskill}, we allocate 10\% for training (24 questions, used for failure detection during evolution), 7\% for validation (17 questions, used for Pareto front selection and regression checks), and the remainder (205 questions) as a held-out test set that is never exposed during evolution. All final numbers in Table~1 are reported exclusively on this held-out set. Stratification ensures that each document category (monthly bulletin, quarterly supplement, annual summary) is represented proportionally across all three splits.

\paragraph{SealQA~\citep{pham2025sealqa}. } We use the \texttt{seal-0} split (111 questions). The same 10\%/7\%/rest ratio yields 11 training, 8 validation, and 92 test questions, stratified by question difficulty tier (easy/medium/hard as annotated by the benchmark authors).

\paragraph{WildClawBench~\citep{Ding_WildClawBench}. } We adopt the day--night cycle protocol of SkillClaw~\citep{ma2026skillclaw}: 6 evolution rounds, each consisting of a daytime phase (8 concurrent simulated users interact with the agent, generating execution trajectories) and a nighttime phase (the system performs reflection, bundle proposal, validation, and deployment). We evaluate on four representative categories---Social Interaction (6 tasks, 60 trials), Search \& Retrieval (11 tasks, 110 trials), Creative Synthesis (11 tasks, 110 trials), and Safety \& Alignment (10 tasks, 100 trials)---where each task is independently tested 10 times per day.

\subsection{Ablation Definitions}\label{app:ablation_defs}
Table~\ref{tab:ablation} and Figure~\ref{fig:resilience} compare the following module configurations.

\noindent\textbf{Skill-only: fixed tools.}
The tool library is held fixed and evolution acts on skills. This removes tool-layer \textsc{Patch}, \textsc{Retype}, \textsc{Extract}, and \textsc{Inline} proposals, testing whether skill edits alone can compensate for tool limitations.

\noindent\textbf{FreeTool: unrestricted tool edits.}
Tool edits are proposed without typed operation templates or their prescribed synchronized update sets.

\noindent\textbf{$-$Ecosys: without the skill ecosystem.}
Disable the ecosystem guidance in Section~\ref{sec:ecosystem}: marginal-utility allocation, interaction-aware retrieval, and posterior-based retirement. Task-based retrieval and the hard execution budget $C$ remain active.

\noindent\textbf{$-$Nogood: without conflict-driven learning.}
Disable persistent failure learning through $\mathcal F_t$: minimization, version-scoped storage with remedies, and reuse for diagnosis and veto. Per-episode localization, reflection, and candidate validation remain active.

\subsection{Initial Skill and Tool Libraries}

Table~\ref{tab:init_skills} and Table~\ref{tab:init_tools} list the complete initial configurations for each benchmark. These are shared identically across all methods to ensure a fair starting point.

\begin{table}[h]
\centering
\caption{Initial Skill library $\mathcal{S}_0$ for each benchmark.}
\label{tab:init_skills}
\small
\setlength{\tabcolsep}{4pt}
\begin{tabular}{lll}
\toprule
\textbf{Benchmark} & \textbf{Skill Name} & \textbf{Referenced Tools} \\
\midrule
\multirow{3}{*}{OfficeQA}
 & \texttt{Doc-Retrieval-Strategy} & \texttt{pdf\_parser} \\
 & \texttt{Table-Parse-Flow} & \texttt{table\_extractor, pdf\_parser} \\
 & \texttt{NumCalc-Protocol} & \texttt{formula\_calc, unit\_converter} \\
\midrule
\multirow{2}{*}{SealQA}
 & \texttt{Search-Plan-Strategy} & \texttt{web\_search, page\_fetcher} \\
 & \texttt{Source-Credibility-Eval} & \texttt{dedup\_ranker} \\
\midrule
\multirow{8}{*}{WildClawBench}
 & \texttt{Slack-Message-Handler} & \texttt{slack\_api} \\
 & \texttt{Calendar-Coordinator} & \texttt{calendar\_api, gmail\_api} \\
 & \texttt{Academic-Search} & \texttt{web\_search} \\
 & \texttt{File-Preflight-Check} & \texttt{fs\_ops} \\
 & \texttt{Creative-Pipeline} & \texttt{img\_processor, video\_processor} \\
 & \texttt{Git-Ops} & \texttt{git\_cli} \\
 & \texttt{Safety-Audit} & \texttt{code\_sandbox} \\
 & \texttt{Report-Generator} & \texttt{pdf\_parser, fs\_ops} \\
\bottomrule
\end{tabular}
\end{table}

\begin{table}[h]
\centering
\caption{Initial Tool library $\mathcal{T}_0$ for each benchmark.}
\label{tab:init_tools}
\small
\setlength{\tabcolsep}{3.4pt}
\begin{tabular}{lll}
\toprule
\textbf{Benchmark} & \textbf{Tool Name} & \textbf{Signature} \\
\midrule
\multirow{4}{*}{OfficeQA}
 & \texttt{pdf\_parser} & \texttt{(filepath, page\_range) $\to$ parsed\_doc} \\
 & \texttt{table\_extractor} & \texttt{(pdf, page, row:int, col:int) $\to$ str} \\
 & \texttt{formula\_calc} & \texttt{(expression:str) $\to$ float} \\
 & \texttt{unit\_converter} & \texttt{(value, from\_unit, to\_unit) $\to$ float} \\
\midrule
\multirow{3}{*}{SealQA}
 & \texttt{web\_search} & \texttt{(query:str, max\_results:int) $\to$ list[result]} \\
 & \texttt{page\_fetcher} & \texttt{(url:str) $\to$ page\_content} \\
 & \texttt{dedup\_ranker} & \texttt{(results:list) $\to$ ranked\_list} \\
\midrule
\multirow{9}{*}{WildClawBench}
 & \texttt{slack\_api} & \texttt{(action, params) $\to$ response} \\
 & \texttt{gmail\_api} & \texttt{(action, params) $\to$ response} \\
 & \texttt{calendar\_api} & \texttt{(action, params) $\to$ response} \\
 & \texttt{web\_search} & \texttt{(query, max\_results) $\to$ list[result]} \\
 & \texttt{fs\_ops} & \texttt{(op, path, content?) $\to$ result} \\
 & \texttt{code\_sandbox} & \texttt{(code:str, lang:str) $\to$ exec\_result} \\
 & \texttt{img\_processor} & \texttt{(image, operation, params) $\to$ image} \\
 & \texttt{video\_processor} & \texttt{(video, operation, params) $\to$ video} \\
 & \texttt{git\_cli} & \texttt{(command:str) $\to$ stdout} \\
\bottomrule
\end{tabular}
\end{table}

\subsection{Scoring Functions}

\paragraph{OfficeQA.} We use the OfficeQA fuzzy-matching scorer at five numerical tolerance levels $\mathcal{Q}=\{0\%,0.1\%,1\%,5\%,10\%\}$. Let $m_{\tau}(x)\in\{0,1\}$ indicate whether the answer passes the scorer at tolerance $\tau$. For evolution, the task score is the fraction of tolerance checks passed, $P(x)=\frac{1}{5}\sum_{\tau\in\mathcal{Q}}m_{\tau}(x)$. Thus, an exact match scores $1.0$; answers whose tightest passing tolerance is $0.1\%$, $1\%$, $5\%$, or $10\%$ score $0.8$, $0.6$, $0.4$, or $0.2$, respectively; answers that fail all checks score $0$. With $\theta=0.8$, a sample enters the failure set when it fails the $0.1\%$ tolerance check. Text-only answers use the scorer's normalized text matching. All accuracy numbers in the main paper use the $0\%$ tolerance setting unless otherwise stated.

\paragraph{SealQA.} We use the official LLM-as-judge protocol: a separate judge model (Qwen3.5-122B, frozen) receives the question, ground-truth answer, and the agent's response, and outputs a binary correct/incorrect judgment. To reduce variance, each response is judged three times and the majority vote is taken. The task score $P(x) \in \{0, 1\}$ is the majority judgment.

\paragraph{WildClawBench.} We follow SkillClaw's multi-indicator aggregation: each task is evaluated against a rubric of 3--5 binary checkpoints (e.g., ``correct recipient'', ``attachment present'', ``deadline extracted''). The task score $P(x)$ is the fraction of checkpoints passed, with a \emph{critical-error override}: if any checkpoint tagged as critical fails, $P(x) = 0$ regardless of other checkpoints. This hard constraint makes WCB particularly sensitive to tool-level reliability, since a single tool failure on a critical step zeroes the entire score.

\subsection{Scaling Factor Definitions}
 
The five task descriptors in the radar chart of Figure~\ref{fig:rq3_gains} (right panel) are computed per-task on WildClawBench at the 397B scale. \textbf{Tool Complexity} is the number of distinct tools invoked in the ground-truth solution trajectory. \textbf{Interaction Density} is the number of skill pairs co-activated during execution, normalized by the total number of possible pairs. \textbf{Context Size} is the total token count of all skill and tool descriptions injected into the agent's context window for that task. \textbf{Action Space} is the number of distinct action types (tool calls, control-flow decisions, output formatting steps) available to the agent at each step, averaged across the trajectory. \textbf{Skill Synergy} is the mean estimated pairwise interaction effect among co-activated skills for that task, measuring how much the activated skill set benefits from joint usage. For each dimension, tasks are binned into terciles (low/medium/high), and the accuracy gain of SkillSmith over the skill-only baseline is averaged within each bin; the radar chart plots the high-tercile gain for each factor.
 
\subsection{Resilience Experiment Setup}
 
The 100-iteration resilience experiment in Figure~\ref{fig:resilience} is conducted on a combined task pool drawn from all three benchmarks using Qwen3.5-122B. The system starts from the standard initial state $\Sigma_0$ and evolves continuously; $-$Ecosys and $-$Nogood denote the module removals in Appendix~\ref{app:ablation_defs}. Two controlled perturbations are introduced to test recovery: at iteration 30, we inject 15 new tasks with higher step counts and more complex tool-chain dependencies than the original distribution, simulating a difficulty escalation in a production deployment; at iteration 60, we randomly corrupt the implementations of 3 out of the current tools in $\mathcal{T}_t$ (replacing function bodies with no-op stubs), simulating tool-layer damage such as an API deprecation or a dependency breakage. All methods share identical perturbation schedules. Library size in Figure~\ref{fig:resilience}b counts the total number of active skills plus active tools ($|\mathcal{S}_t| + |\mathcal{T}_t|$) at each iteration.

\section{Empirical Evolution Cost}
\label{app:cost}

Table~\ref{tab:cost} reports actual costs measured on OfficeQA at the 122B scale. We report backbone model calls and the total number of input and output tokens.
 
\begin{table}[h]
\centering
\caption{Cost comparison on OfficeQA (Qwen3.5-122B).}
\label{tab:cost}
\small
\setlength{\tabcolsep}{4pt}
\begin{tabular}{lccccc}
\toprule
\textbf{Method} & \textbf{Calls/iter} & \textbf{Iters} & \textbf{Total calls} & \textbf{Total tokens (M)} & \textbf{Accuracy} \\
\midrule
EvoSkill & 42.6 & 18 & 766.8 & 18.4 & 63.9\% \\
SkillSmith & 51.3 & 13 & 666.9 & 16.0 & 72.2\% \\
\bottomrule
\end{tabular}
\end{table}
 
SkillSmith uses 51.3 model calls per iteration over 13 iterations, compared with EvoSkill's 42.6 calls over 18 iterations. This corresponds to approximately 13\% fewer total calls (667 vs.\ 767) and 13\% fewer total tokens (16.0M vs.\ 18.4M), with an accuracy gain of 8.3 percentage points.

\section{Related Work}
\label{app:related_work}

\paragraph{Skill libraries and reflective evolution.}
Trace2Skill consolidates execution experience into transferable skills \citep{ni2026trace2skill}; SkillX builds a hierarchical skill knowledge base with feedback-driven refinement and exploratory expansion \citep{wang2026skillx}. EvoSkill combines failure analysis with Pareto selection to discover and revise skill folders under a frozen model \citep{alzubi2026evoskill}, while SkillClaw aggregates experience across users and time \citep{ma2026skillclaw}. GEPA uses trajectory-level reflection to propose prompt updates and combine complementary candidates \citep{agrawal2025gepa}. These approaches establish external artifacts as a practical target for adaptation. SkillSmith studies how such artifacts share a finite execution window: randomized masking estimates skill effects and interactions, and ecosystem dynamics allocate context according to marginal utility per token. SkillSmith couples library allocation and modification under both budgets.

\paragraph{Executable skills and reusable abstractions.}
Toolformer learns API invocation through self-supervised training, including when to call tools and how to use their outputs \citep{schick2023toolformer}. Voyager maintains a library of executable skills and improves programs using environment feedback and execution errors \citep{wang2023voyager}. DreamCoder learns reusable program components by identifying common structure across solutions, jointly improving its library and search policy \citep{ellis2021dreamcoder}. Thus, executable skill reuse and abstraction discovery already have substantial precedents. SkillSmith connects abstraction extraction to context allocation: moving shared, testable logic behind a tool interface lowers the context cost of its callers. Its four typed operations specify both the tool change and the required caller updates, allowing interface changes to be validated and committed with their dependents. The capacity result formalizes the benefit of lower costs within the paper's allocation model.

\paragraph{Context management and resource allocation.}
MemGPT manages a hierarchy of memory tiers, moving information between external storage and a bounded context window \citep{packer2023memgpt}. This makes the distinction between persistent storage and active context explicit: a large external memory need not be loaded in full. SkillSmith applies this distinction to a library of executable capabilities and workflows. Its allocation variables govern which skills enter the window, while pairwise effects capture how their usefulness depends on co-activation. Tool extraction changes the cost of the assets being allocated. Memory management and abstraction extraction are therefore complementary ways of supporting a larger persistent library within a fixed execution window.

\paragraph{Evolutionary dynamics and online learning.}
Replicator dynamics provides a classical model of frequency-dependent selection \citep{hofbauer1998evolutionary}; multiplicative-weight updates allocate resources among alternatives in online learning \citep{freund1997decision}. SkillSmith uses these foundations to evolve reserved context shares, with fitness defined by the marginal utility of additional tokens. Bayesian estimation supplies uncertainty-aware fitness estimates from randomized skill masking. This interpretation connects competition between skills to a measurable resource constraint. The continuous-time ascent analysis holds for a fixed utility model; sampled updates and evolving libraries introduce estimation and model changes. Accordingly, the theory characterizes the allocation mechanism under stated assumptions, rather than guaranteeing monotonic improvement throughout adaptive search.

\paragraph{Failure memory and conflict learning.}
Reflexion stores verbal feedback in episodic memory to guide subsequent trials without updating model weights \citep{shinn2023reflexion}. In software debugging, delta debugging reduces a failure-inducing input to a 1-minimal subset under its testing assumptions \citep{zeller2002simplifying}. In satisfiability solving, GRASP records conflict causes to preempt repeated conflicts \citep{marquessilva1999grasp}. SkillSmith combines these ideas at the level of proposed edits: replay tests isolate a failing modification subset, and version-scoped nogoods constrain later proposals. This connects retrospective feedback to prospective search pruning. Unlike clauses derived from a fixed logical formula, these records depend on component versions and empirical replay; exact-match veto therefore requires the stated persistence assumption, while similarity-based matching remains heuristic. The shared iteration budget charges both validation and failure minimization, making the cost of learning a reusable constraint explicit throughout the evolution of skills, tools, and their dependencies.

\end{document}